\documentclass{article}
\usepackage{natbib}
\usepackage{amsmath}
\usepackage{amsfonts}
\usepackage{bbm}
\usepackage{mathabx}
\usepackage{csquotes}
\usepackage{amsthm}
\usepackage{dsfont}
\usepackage{enumerate}
\usepackage[OT4]{fontenc}
\usepackage{graphicx}
\usepackage{enumitem}
\usepackage{multirow}
\usepackage{hyperref}
\usepackage{rotating}
\usepackage{float}
\usepackage{booktabs}
\usepackage[ruled]{algorithm2e}
\def\ud{\, \mathrm{d}}
\DeclareMathOperator*{\argmax}{argmax}
\DeclareMathOperator*{\argmin}{argmin}
\DeclareMathOperator{\sgn}{sgn}

\newtheorem{theorem}{Theorem}
\newtheorem{corollary}[theorem]{Corollary}
\newtheorem{definition}[theorem]{Definition}

\newtheorem{lemma}[theorem]{Lemma}
\newtheorem{remark}[theorem]{Remark}

\numberwithin{theorem}{section}
\numberwithin{equation}{section}

\begin{document}




\title{Covariance-based Dissimilarity Measures Applied to Clustering Wide-sense Stationary Ergodic Processes}






\author{Qidi Peng\footnote{Institute of Mathematical Sciences, Claremont Graduate University, Claremont, CA 91711. Email: qidi.peng@cgu.edu.} \and Nan Rao\footnote{School of Mathematical Sciences, Shanghai Jiao Tong University, Shanghai, China. Email: nan.rao@sjtu.edu.cn.} \and Ran Zhao\footnote{Institute of Mathematical Sceinces and Drucker School of Management, Claremont Graduate University, Claremont, CA 91711. Email: ran.zhao@cgu.edu.}}

\date{}

\maketitle

\begin{abstract}
We introduce a new unsupervised learning problem: clustering wide-sense stationary ergodic stochastic processes. A covariance-based dissimilarity measure together with asymptotically consistent algorithms is designed for clustering offline and online datasets, respectively. We also suggest a formal criterion on the efficiency of dissimilarity measures, and discuss of some approach to improve the efficiency of our clustering algorithms, when they are applied to cluster particular type of processes, such as self-similar processes with wide-sense  stationary ergodic increments. Clustering synthetic data and real-world data are provided as examples of applications.

\begin{flushleft}

\textbf{MCS (2010): } 62-07 $\cdot$  60G10 $\cdot$ 62M10
\end{flushleft}


\end{abstract}



\newpage
\section{Introduction}
\label{introduction}
Cluster analysis, as a core category of unsupervised learning techniques, allows to discover hidden patterns in data where one does not know the true answer upfront.  Its goal is to assign a heterogeneous set of objects into non-overlapping clusters, where in each cluster any two objects are more related to each other than to objects in other clusters. Given its exploratory nature, clustering has nowadays a number of applications in various fields of both industry and scientific research, such as biological and medical research \citep{Damian2007,Zhao2014,Jaaskinen2014}, information technology \citep{Jain1999,Slonim2005},  signal and image processing \citep{Rubinstein2013}, geology \citep{Juo2001} and finance \citep{Pavlidis2006,Bastos2014,Ieva2016}. There exists a rich literature of cluster analysis on random vectors, where the objects, waiting to be clustered, are sampled from high-dimensional joint distributions. There is no shortage of such clustering algorithms \citep{Xu2005}. However, stochastic processes are quite a different setting from random vectors, since their observations (sample paths) are sampled from processes distributions. While the cluster analysis on random vectors has developed aggressively, clustering on stochastic processes receives much less attention. Today cluster analysis on stochastic processes deserves increasingly intense study, thanks to their vital importance to many applied areas, where the collected information are indexed by real time and are especially long. Examples of these time-indexed information include biological data, financial data, marketing data, surface weather data, geological data and video/audio data, etc. 

Recall that in the setting of random vectors, a process of clustering often consists of two steps:
\begin{description}
\item[\textbf{Step 1}] One suggests a suitable dissimilarity measure to describe the distance between 2 objects, under which \enquote{two objects are close to each other} becomes meaningful.
\item[\textbf{Step 2}] One designs an enough accurate and computationally efficient clustering function based on the above dissimilarity measure.
\end{description}
Clustering stochastic processes is performed in a similar way but new challenges may arise in both Step 1 and Step 2. Intuitively, one can always apply existing  random vectors clustering approaches to cluster arbitrary stochastic processes, such as non-hierarchical approaches ($K$-means clustering methods) and hierarchical approaches (agglomerative method, divisive method) \citep{Hartigan1975}, based on \enquote{naive} dissimilarity measures (e.g., Euclidean distance, Manhattan distance or Minkowski distance). However, one faces at least 2 potential risks when applying the above approaches to clustering stochastic processes:
\begin{description}
\item[\textbf{Risk 1}] These approaches might suffer from their huge complexity costs, due to the great length of their sample paths. As a result classical clustering algorithms are often computationally forbidding \citep{Ieva2016,Peng2008}.
\item[\textbf{Risk 2}] These approaches might suffer from over-fitting issues. For example, clustering stationary or periodic processes based on \textit{Euclidean distance} between the paths, without considering their path properties will result in \enquote{over fitting, bad clusters} situation.
\end{description}
In summary, classical dissimilarity measures or clustering strategies would fail in clustering stochastic processes.

Fortunately, the complexity cost and the over-fitting errors of clustering processes could be largely reduced, if one is aware of the fact that a stochastic process often possesses fine paths features (e.g., stationarity, Markov property, self-similarity, sparsity, seasonality, etc.), which is unlike an arbitrary random vector. An appropriate dissimilarity measure then should be chosen to be able to capture these paths features. Clustering processes is then performed to group any two sample paths into one group, if they are relatively close to each other under that particular dissimilarity measure.  Below are some examples provided in the literature.

\cite{Peng2008} proposed a dissimilarity measure between two special sample paths of processes. In their setting it is supposed that, for each path only sparse and irregularly spaced measurements with additional measurement errors are available. Such features occur commonly in longitudinal studies and online trading data. Based on this particular dissimilarity measure, classification and cluster analysis could be made. \cite{Ieva2016} developed a new algorithm to perform clustering of multivariate and functional data, based on a covariance-based dissimilarity measure. Their attention is focused on the specific case of a set of observations from two populations, whose probability distributions have equal mean but differ in terms of covariances. \cite{khaleghi2016} designed consistent algorithms for clustering strict-sense stationary ergodic processes (see the forthcoming Eq. (\ref{strictergodic}) for the definition of strict-sense ergodicity), where the dissimilarity measure is proposed as distance of process distributions. It is worth noting that the consistency of their algorithms is guaranteed thanks to the assumption of strict-sense ergodicity.

In this framework, we aim to design asymptotically consistent algorithms to cluster a general class of stochastic processes, i.e., wide-sense stationary ergodic processes (see Definition \ref{process} below). Asymptotically consistent algorithms can be obtained for this setting, since the covariance stationarity and ergodicity allow the process to present some featured asymptotic behavior with respect to their length, rather than to the total number of paths. 
\begin{definition}[Wide-sense stationary ergodic process]
\label{process}
A stochastic process $X=\{X_t\}_{t\in T}$ (the time indexes set $T$ can be either $\mathbb R_+=[0,+\infty)$ or $\mathbb N=\{1,2\ldots\}$) is called wide-sense stationary if its mean and covariance structure are finite and time-invariant: $\mathbb E(X_t)=\mu$ for any $t\in T$, and for any subset $(X_{i_1},\ldots,X_{i_r})$, its covariance matrix remains invariant subject to any time shift $h>0$:
$$
\mathbb Cov(X_{i_1},\ldots,X_{i_r})= \mathbb Cov(X_{i_1+h},\ldots,X_{i_r+h}).
$$
Denote by $\gamma$ the auto-covariance function of $X$. Then $X$ is further called weakly ergodic (or wide-sense ergodic) if it is ergodic for the mean and the second-order moment:
\begin{itemize}
\item If $X$ is a continuous-time process (e.g., $T=\mathbb R_+$), then  it satisfies for any $s\in \mathbb R_+$, 
$$
\frac{1}{h}\int_s^{s+h}X_u\ud u\xrightarrow[h\to+\infty]{a.s.}\mu,
$$
and
$$\frac{1}{h}\int_s^{s+h}(X_{u+\tau}-\mu)(X_u-\mu)\ud u\xrightarrow[h\to+\infty]{a.s.}\gamma(\tau),~\mbox{for all  $\tau\in \mathbb R_+$},
$$
where $\xrightarrow[]{a.s.}$ denotes the almost sure convergence (convergence with probability $1$).
\item If $X$ is a discrete-time process (e.g., $T=\mathbb N$), then  it satisfies for any $s\in \mathbb N\cup\{0\}$, 
$$
\frac{X_s+X_{s+1}+\ldots+X_{s+h}}{h+1}\xrightarrow[h\in\mathbb N,~h\to+\infty]{a.s.}\mu,
$$
and
$$\frac{\sum_{u=s}^{s+h}(X_{u+\tau}-\mu)(X_u-\mu)}{h+1}\xrightarrow[h\in\mathbb N,~h\to+\infty]{a.s.}\gamma(\tau),~\mbox{for all $\tau\in \mathbb N\cup\{0\}$}.
$$
\end{itemize}
\end{definition}
Wide-sense stationarity and ergodicity are believed to be a very general assumption, at least in the following senses:
\begin{enumerate}
\item The assumption that each process is generated by some mean and covariance structure is sufficient for capturing all features of a wide-sense stationary ergodic process. In other words, our algorithms intend to cluster means and auto-covariance functions, not process distributions. 
\item Wide-sense stationary ergodic process partially extends the strict-sense one. A finite-variance strict-sense stationary ergodic process (see Eq. (\ref{strictergodic}) for its definition) is also wide-sense stationary ergodic. However strict-sense stationary ergodic stable processes are not wide-sense stationary, because their variances explode \citep{Cambanis1987,Samorodnitsky2004}.
\item A Gaussian process can be \textit{fully identified only} by its mean and covariance structure. Then a wide-sense stationary  ergodic Gaussian process is also strict-sense stationary ergodic.
\item In the clustering problem, the dependency among the sample paths can be arbitrary. 
\end{enumerate}
There is a long list of processes which are wide-sense stationary ergodic, but not necessarily stationary in the strict sense. The examples of wide-sense stationary processes below are not exhausted.
\begin{description}
\item[\textbf{Example 1}] Non-independent White Noise.

Let $U$ be a random variable uniformly distributed over $(0,2\pi)$ and define
$$
Z(t):=\sqrt{2}\cos(tU),~\mbox{for}~t\in\mathbb N.
$$
The process $Z=\{Z(t)\}_{t\in\mathbb N}$ is then a white noise because it verifies 
$$
\mathbb E(Z(t))=0,~\mathbb Var(Z(t))=1~\mbox{and}~\mathbb Cov(Z(s),Z(t))=0,~\mbox{for}~s\neq t.
$$
We claim that $Z$ is wide-sense stationary ergodic, which can be obtained by using the Kolmogorov's strong law of large numbers, see e.g. Theorem 2.3.10 in \cite{Sen1993}. However $Z$ is not strict-sense stationary since
$$
(Z(1),Z(2))\neq (Z(2),Z(3))~\mbox{in law}.
$$
Indeed, it is easy to see that
$$
0<\mathbb E\big(Z(1)^2Z(2)\big)\neq\mathbb E\big(Z(2)^2Z(3)\big)=0.
$$
\item[\textbf{Example 2}] Auto-regressive Models.

It is well-known that an auto-regressive model $\{Y(t)\}_t\sim AR(1)$ in the form:
\begin{equation}
\label{AR1}
Y(t)=aY(t-1)+Z(t),~|a|<1,a\neq 0,~\mbox{for}~t\in\mathbb N
\end{equation}
is wide-sense stationary ergodic. However it is not necessarily strict-sense stationary ergodic, when the joint distributions of the white noise $\{Z(t)\}_t$ are not invariant with time-shifting (e.g., take $\{Z(t)\}_t$ to be the white noise in Example 1).
\item[\textbf{Example 3}] Increment Process of Fractional Brownian Motion.

Let $\{B^H(t)\}_t$ be a fractional Brownian motion with Hurst index $H\in(0,1)$ (see \cite{Mandelbrot1968}). For each $h>0$, its increment process $\{Z^h(t):=B^H(t+h)-B^H(t)\}_t$ is finite-variance strict-sense stationary ergodic \citep{Magdziarz2011}. As a result it is also wide-sense stationary ergodic. More detail will be discussed in Section \ref{sec::exper_results}.

\item[\textbf{Example 4}] Increment Process of More General Gaussian Processes.

\cite{Peng2012} introduced a general class of zero-mean Gaussian processes $X=\{X(t)\}_{t\in\mathbb R}$ having stationary increments. Its variogram $\nu(t):=2^{-1}\mathbb E(X(t)^2)$ satisfies:
\begin{description}
\item[(1)] There is a non-negative integer $d$ such that $\nu$ is $2d$-times continuously differentiable over $[-2,2]$, but not $2(d+1)$-times continuously differentiable over $[-2,2]$.
\item[(2)] There are 2 real numbers $c\neq 0$ and $s_0\in(0,2)$, such that for all $t\in[-2,2]$,
$$
\nu(t)=\nu^{(2d)}(0)+c|t|^{s_0}+r(t),
$$
where the remainder $r(t)$ satisfies:
\begin{itemize}
\item $r(t)=o(|t|^{s_0})$, as $t\to0$.
\item There are two real numbers $c' > 0$, $\omega > s_0$ and an integer $q > \omega+1/2$ such that $r$ is $q$-times continuously differentiable on $[-2, 2]\backslash\{0\}$ and for all $t \in [-2, 2]\backslash\{0\}$, we have
$$
|r^{(q)}(t)|\le c'|t|^{\omega-q}.
$$
\end{itemize}
\end{description}
It is shown that the process $X$ extends fractional Brownian motion and it also has wide-sense (and strict-sense) stationary ergodic increments when $d+s_0/2\in(0,1)$ (see Proposition 3.1 in \cite{Peng2012}).
\end{description}
The problem of clustering processes via their means and covariance structures leads us to formulating our clustering targets in the following way.
\begin{definition}[Ground-truth $G$ of covariance structures]
\label{ground-truth}
Let 
$$
G=\big\{G_1,\ldots,G_\kappa\big\}
$$
be a partitioning of $\mathbb N=\{1,2,\ldots\}$ into $\kappa$ disjoint sets $G_k$, $k=1,\ldots,\kappa$, such that the means and covariance structures of $\mathbf x_i$, $i\in \mathbb N$ are identical, if and only if $i\in G_k$ for some $k=1,\ldots,\kappa$. Such $G$ is called \textit{ground-truth of covariance structures}. We also denote by $G|_{N}$ the restriction of $G$ to the first $N$ sequences:
$$
G|_{N}=\big\{G_k\cap \{1,\ldots,N\}:~k=1,\ldots,\kappa\big\}.
$$
\end{definition}
Our clustering algorithms will aim to output the ground-truth partitioning $G$, as the sample length grows. Before stating these algorithms, we introduce the inspiring framework done by \citet{khaleghi2016}.

\subsection{Preliminary Results: Clustering Strict-sense Stationary Ergodic Processes}
\label{Preliminary_Results}
\cite{khaleghi2016} considered the problem of clustering strict-sense stationary ergodic processes. The main fruit in \cite{khaleghi2016} is obtaining the so-called asymptotically consistent algorithms to cluster processes of that type. We briefly state their work below. Depending on how the information is collected, the stochastic processes clustering problems consist of dealing with two models: offline setting and online setting.
\begin{description}
\item[\textbf{Offline setting:}] The observations are assumed to be a finite number $N$ of paths: 
$$
\mathbf{x}_1 = \Big(X_1^{(1)},\ldots, X_{n_1}^{(1)}\Big),\ldots,\mathbf{x}_N = \Big(X_1^{(N)},\ldots, X_{n_N}^{(N)}\Big).
$$
Each path is generated by one of the $\kappa$ different unknown process distributions. In this case, an asymptotically consistent clustering function should satisfy the following.
\begin{definition}[Consistency: offline setting]
\label{consistency:offline}
A clustering function $f$ is consistent for a set of sequences $S$ if $f(S,\kappa)=G$. Moreover, denoting $n=\min\{n_1,\ldots,n_N\}$, $f$ is called strongly asymptotically consistent in the offline sense if with probability $1$ from some $n$ on it is consistent on the set $S$, i.e.,
$$
\mathbb P\left(\lim_{n\to\infty}f(S,\kappa)=G\right)=1.
$$
It is called weakly asymptotically consistent if $\lim\limits_{n\to\infty}\mathbb P(f(S,\kappa)=G)=1$.
\end{definition}
\item[\textbf{Online setting:}] In this setting the observations, having growing length and number of scenarios with respect to time $t$, are denoted by 
$$
\mathbf{x}_1 = \Big(X_1^{(1)},\ldots, X_{n_1}^{(1)}\Big),\ldots,\mathbf{x}_{N(t)} = \Big(X_1^{(N(t))},\ldots, X_{n_{N(t)}}^{(N(t))}\Big),
$$
where the index function $N(t)$ is non-decreasing with respect to $t$. 

Then an asymptotically consistent online clustering function is defined below:
\begin{definition}[Consistency: online setting]
\label{consistency:online}
A clustering function is strongly (RESP. weakly) asymptotically consistent in the online sense, if for every $N\in\mathbb N$ the clustering $f(S(t),\kappa)|_N$ is strongly (RESP. weakly) asymptotically consistent in the offline sense, where $f(S(t),\kappa)|_N$ is the clustering $f(S(t),\kappa)$ restricted to the first $N$ sequences:
$$
f(S(t),\kappa)|_N=\left\{f(S(t),\kappa)\cap \{1,\ldots,N\}:~k=1,\ldots,\kappa\right\}.
$$
\end{definition}
\end{description}
There is a detailed discussion on the comparison of offline and online settings in \cite{khaleghi2016}, stating that these two settings have significant differences, since using the offline algorithm in the online setting by simply applying it to the entire data observed at every time step, does not result in an asymptotically  consistent algorithm. Therefore separately and independently studying these two settings becomes necessary and meaningful. 

As the main results in \cite{khaleghi2016}, asymptotically consistent clustering algorithms for both offline and online settings are designed. They are then successfully applied to clustering synthetic and real data sets.

Note that in the framework of \cite{khaleghi2016}, a key step is introduction to the so-called \textit{distributional distance} \citep{gray1988}: the distributional distance between a pair of process distributions $\rho_1$, $\rho_2$ is defined to be
\begin{equation}
\label{d1}
d(\rho_1,\rho_2)=\sum_{m,l=1}^{\infty}w_m w_l\sum_{B\in B^{m,l}}\left|\rho_1(B)-\rho_2(B)\right|,
\end{equation}
where:
\begin{itemize}
\item The sets $B^{m,l}$, $m,l\ge1$ are obtained via the partitioning of $\mathbb R^m$ into cubes of dimension $m$ and volume $2^{-ml}$, starting at the origin.
\item The sequence of weights $\{w_j\}_{j\ge1}$ is positive and decreasing to zero. Moreover it should be chosen such that the series in (\ref{d1}) is convergent. The weights are often suggested to give precedence to earlier clusterings, protecting the clustering decisions from the presence of the newly observed sample paths, whose corresponding distance estimates may not yet be accurate. For instance, it is set to be $w_j=1/j(j+1)$ in \cite{khaleghi2016}.
\end{itemize}
Further, the distance between two sample paths $\mathbf x_1$, $\mathbf x_2$ of stochastic processes is given by
\begin{equation}
\label{hatd1}
\widehat d(\mathbf x_1,\mathbf x_2)=\sum_{m=1}^{m_n}\sum_{l=1}^{l_n}w_m w_l\sum_{B\in B^{m,l}}|\nu(\mathbf x_1,B)-\nu(\mathbf x_2,B)|,
\end{equation}
where:
\begin{itemize}
\item $m_n,l_n$ ($\le n$) can be arbitrary sequences of positive integers  increasing to infinity, as $n\to\infty$. 
\item For a process path  $\mathbf x=(X_1,\ldots,X_n)$, and an event $B$, $\nu(\mathbf x,B)$ denotes the average times that the event $B$ occurs over $n-m_n+1$ time intervals. More precisely,
$$
\nu(\mathbf x,B):=\frac{1}{n-m_n+1}\sum_{i=1}^{n-m_n+1}\mathds 1\{(X_i,\ldots,X_{i+m_n-1})\in B\}.
$$
\end{itemize}
The process distribution $X$ from which $\mathbf x$ is sampled is called strictly ergodic if
\begin{equation}
\label{strictergodic}
\mathbb P\left(\lim_{n\to\infty}\nu(\mathbf x,B)=\mathbb P(X\in B)\right)=1,~\mbox{for all $B$}.
\end{equation}
The assumption that the processes are ergodic leads to that $\widehat d$ is a strongly consistent estimator of $d$: 
$$
\mathbb P\left(\lim_{n\to\infty}\widehat d(\mathbf x_1,\mathbf x_2)=d(\rho_1,\rho_2)\right)=1,
$$
where $\rho_1,\rho_2$ are the process distributions corresponding to $\mathbf x_1,\mathbf x_2$, respectively.

Based on the distances $d$ and their estimates $\widehat d$, the asymptotically consistent algorithms for clustering stationary ergodic processes in each of the offline and online settings are provided (see Algorithms 1, 2 and Theorems 11, 12 in \cite{khaleghi2016}). \cite{khaleghi2016} also show that their methods can be implemented efficiently: they are at most quadratic in each of their arguments, and are linear (up to log terms) in some formulations.

\subsection{Statistical Setting: Clustering Wide-sense Stationary Ergodic Processes}
Inspired by the framework of \cite{khaleghi2016}, we consider the problem of clustering \textit{wide-sense stationary ergodic processes}.  We first introduce the following \textit{covariance-based dissimilarity measure}, which is one of the main contributions of this paper. 
\begin{definition}{(Covariance-based dissimilarity measure)}
\label{definition1}
The covariance-based dissimilarity measure $d^*$ between a pair of processes $X^{(1)}$, $X^{(2)}$ (in fact $X^{(1)}$, $X^{(2)}$ denote two covariance structures, each may contain different process distributions) is defined as follows:
\begin{eqnarray}
\label{def:d*}
&&d^*\big(X^{(1)},X^{(2)}\big):= \sum_{m,l = 1}^{\infty} w_m w_l\nonumber\\
&&\times\mathcal M\left(
\left(\mathbb E\big(X_{l\ldots l+m-1}^{(1)}\big),
\mathbb Cov\big(X_{l\ldots l+m-1}^{(1)}\big)\right), 
\left(
\mathbb E\big(X_{l\ldots l+m-1}^{(2)}\big), \mathbb Cov\big(X_{l\ldots l+m-1}^{(2)}\big)\right)\right),\nonumber\\
\end{eqnarray}
where:
\begin{itemize}
\item For $j=1,2$, $\{X_l^{(j)}\}_{l\in\mathbb N}$ denotes some path sampled from the process $X^{(j)}$. We assume that all possible observations of the process $X^{(j)}$ is a subset of $\{X_l^{(j)}\}_{l\in\mathbb N}$. For $l'\ge l\ge 1$, we define the shortcut notation  $X_{l\ldots l'}^{(j)}:=(X_{l}^{(j)},\ldots,X_{l'}^{(j)})$.
\item The function $\mathcal M$ is defined by: for any $p_1,p_2,p_3\in\mathbb N$, any 2 vectors $v_1,v_2\in\mathbb R^{p_1}$ and any 2 matrices $A_1,A_2\in \mathbb R^{p_2\times p_3}$,
\begin{equation}
\label{M}
\mathcal M((v_1,A_1),(v_2,A_2)):= \left|v_1-v_2\right|+\rho^*\left(A_1,A_2\right).
\end{equation}
\item The distance $\rho^*$ between 2 equal-sized matrices $M_1,M_2$ is defined to be
\begin{equation}
\label{def:rho}
\rho^*(M_1,M_2):=\|M_1-M_2\|_F,
\end{equation}
with $\|\cdot\|_F$ being the \textit{Frobenius norm}:

for an arbitrary matrix $M=\{M_{ij}\}_{i=1,\ldots,m; j=1,\ldots,n}$,
$$
\|M\|_F:=\sqrt{\sum_{i=1}^m\sum_{j=1}^n|M_{ij}|^2}.
$$
Introduction to the matrices distance $\rho^*$ is inspired by \cite{herdin2005}. The matrices distance given in \cite{herdin2005} is used to measure the distance between 2 correlation matrices. However, our distance $\rho^*$ is a modification of the one in the latter paper. Indeed, unlike \cite{herdin2005}, $\rho^*$ is a well-defined metric distance, as it satisfies the triangle inequalities. 
\item The sequence of positive weights $\{w_j\}$ is chosen such that $d^*(X^{(1)},X^{(2)})$ is finite. Observe that the distances $|\cdot|$ and $\rho^*$ in (\ref{def:d*})  do not depend on $l$, as a result we necessarily have
\begin{equation}
\label{sum_weights}
\sum_{l=1}^\infty w_l<+\infty.
\end{equation}
In practice a typical choice of weights we suggest is $w_j=1/j(j+1)$, $j=1,2,\ldots$. This is because, for most of the well-known covariance stationary ergodic processes (causal $ARMA(p,q)$, increments of fractional Brownian motions, etc.), their auto-covariance functions are absolutely summable: denote by $\gamma_X$ the auto-covariance function of $\{X_t\}_{t}$,
\begin{equation}
\label{ergodic_condition}
\sum_{h=-\infty}^{+\infty}\left|\gamma_X(h)\right|<+\infty.
\end{equation}
\cite{Slezak2017} pointed out that (\ref{ergodic_condition}) is a sufficient condition for $\{X_t\}$ being mean-ergodic. However (\ref{ergodic_condition}) does not necessarily imply that $\{X_t\}$ is covariance-ergodic. It becomes a sufficient and necessary condition if $\{X_t\}$ is Gaussian. 
Therefore subject to (\ref{ergodic_condition}), taking $w_j=1/j(j+1)$, we obtain for any integer $N>0$,
\begin{eqnarray}
\label{ineq_1}
&&\sum_{m,l = 1}^{N} w_m w_l \left|\mathbb E\left(X_{l\ldots l+m-1}^{(1)}\right)-\mathbb E\left(X_{l\ldots l+m-1}^{(2)}\right)\right|\nonumber\\
&&= \sum_{m,l = 1}^{N} w_m w_l \sqrt{m}|\mu_1-\mu_2|=|\mu_1-\mu_2|\sum_{m,l = 1}^{N}\frac{1}{\sqrt{m}(m+1)l(l+1)}\nonumber\\
&&\le |\mu_1-\mu_2|\sum_{m,l = 1}^{+\infty}\frac{1}{\sqrt{m}(m+1)l(l+1)}<+\infty,
\end{eqnarray}
with $\mu_j=\mathbb E\big(X_1^{(j)}\big)$, for $j=1,2$; and
\begin{eqnarray}
\label{ineq_2}
&&\sum_{m,l = 1}^{N} w_m w_l \rho^*\left(\mathbb Cov\Big(X_{l\ldots l+m-1}^{(1)}\Big),\mathbb Cov\Big(X_{l\ldots l+m-1}^{(2)}\Big)\right)\nonumber\\
&&\le \sum_{m,l = 1}^{N} w_m w_l \sqrt{2\sum_{k_1=1}^m\sum_{k_2=1}^m\left(\gamma_X(|k_1-k_2|)\right)^2}\nonumber\\
&&= \sum_{m,l = 1}^{N} w_m w_l \sqrt{2\sum_{q=-(m-1)}^{m-1}(m-|q|)\left(\gamma_X(|q|)\right)^2}\nonumber\\
&&\le \sum_{m,l = 1}^{N} w_m w_l \sqrt{2m\sum_{q=-(m-1)}^{m-1}\left(\gamma_X(|q|)\right)^2}\nonumber\\
&&\le c\sum_{m,l = 1}^{N} \frac{\sqrt{2m}}{m(m+1)l(l+1)}\le c\sum_{m,l = 1}^{+\infty} \frac{\sqrt{2m}}{m(m+1)l(l+1)}<+\infty,
\end{eqnarray}
where the constant $c=\sum_{q=-\infty}^{\infty}|\mathbb Cov(X_1,X_{1+|q|})|<+\infty$. Therefore combining (\ref{ineq_1}) and (\ref{ineq_2}) leads to
$$
d^*\big(X^{(1)},X^{(2)}\big)<+\infty.
$$
Hence $d^*(X^{(1)},X^{(2)})$ in (\ref{def:d*}) is well-defined.
\end{itemize}
\end{definition}
In (\ref{def:d*}) and (\ref{M}) we see that the behavior of the dissimilarity measure $d^*$ is jointly explained by the Euclidean distance of means and the matrices distance of covariance matrices. If the means of the processes $X^{(1)}$ and $X^{(2)}$ are priorly known to be equal, the distance $d^*$ can be simplified to:
\begin{equation}
\label{def:d*_bis}
d^*\big(X^{(1)},X^{(2)}\big)= \sum_{m,l = 1}^{\infty} w_m w_l\rho^*\left(
\mathbb Cov\big(X_{l\ldots l+m-1}^{(1)}\big), \mathbb Cov\big(X_{l\ldots l+m-1}^{(2)}\big)\right).
\end{equation}
Note that this dissimilarity measure can be applied on self-similar processes, since they are all zero-mean (see Section \ref{sec:log}). 

Next we provide consistent estimator of ${d^*}(X^{(1)},X^{(2)})$.
For $1\le l\le n$ and $m\le n-l+1$, define $\mu^*({X_{l\ldots n}}, m)$ to be the empirical mean of a process $X$'s sample path $(X_l,\ldots,X_n)$:  
\begin{equation}
\label{mu}
\mu^*(X_{l\ldots n},m) :=\frac{1}{n-m-l+2}
\sum_{i = l} ^{n-m+1} (X_i~\ldots~X_{i+m-1})^T,
\end{equation}
and define $\nu^*({X_{l\ldots n}}, m)$ to be the empirical covariance matrix of $(X_l,\ldots,X_n)$:  
\begin{eqnarray}
\label{nu}
\nu^*(X_{l\ldots n},m) &:=& \frac{1}{n-m-l+2}
\sum_{i = l} ^{n-m+1} (X_i~\ldots~X_{i+m-1})^T(X_i~\ldots~X_{i+m-1})\nonumber\\
&&-\mu^*(X_{l\ldots n},m)\mu^*(X_{l\ldots n},m)^T,
\end{eqnarray}
where $M^T$ denotes the transpose of the matrix $M$.

Recall that the notion of wide-sense ergodicity is given in Definition \ref{process}. The ergodicity theorem concerns what information can be derived from an average over time about the ensemble average at each point of time. For the wide-sense stationary ergodic process $X$, being either  continuous-time or discrete-time, the following statement holds: every empirical mean $\mu^*(X_{l\ldots n},m)$ is a strongly consistent estimator of the path mean $\mathbb E(X_{l\ldots l+m-1})$; and every empirical covariance matrix $\nu^*(X_{l\ldots n},m)$ is a strongly consistent estimator of the covariance matrix $\mathbb Cov(X_{l\ldots l+m-1})$ under the Frobenius norm, i.e., for all $m\ge1$, we have
\begin{equation*}
\mathbb P\left(\lim_{n\rightarrow \infty}\left|\mu^*(X_{l\ldots n},m)-\mathbb E(X_{l\ldots l+m-1})\right|=0\right)=1
\end{equation*}
and
\begin{equation*}
\mathbb P\left(\lim_{n\rightarrow \infty}\left\|\nu^*(X_{l\ldots n},m)-\mathbb Cov(X_{l\ldots l+m-1})\right\|_F=0\right)=1.
\end{equation*}
Next we introduce the empirical covariance-based dissimilarity measure $\widehat{d^*}$, serving as a consistent estimator of the covariance-based dissimilarity measure $d^*$.
\begin{definition}[Empirical covariance-based dissimilarity measure]
Given two processes' sample paths $\mathbf x_j=(X_1^{(j)},\ldots,X_{n_j}^{(j)})$, $j=1,2$. Let $n=\min\{n_1,n_2\}$, we define the empirical covariance-based dissimilarity measure between $\mathbf x_1$ and $\mathbf x_2$  by
\begin{eqnarray}
\label{dxx}
&&\widehat{d^*}(\mathbf x_{1},\mathbf x_{2}):=\sum_{m= 1}^{m_n} \sum_{l= 1}^{n-m+1} w_m w_l\nonumber\\
&&\times\mathcal M\left( \left(\mu^*(X^{(1)}_{l\ldots n},m), \nu^*(X^{(1)}_{l\ldots n},m)\right),\left(\mu^*(X^{(2)}_{l\ldots n},m), \nu^*(X^{(2)}_{l\ldots n},m)\right)\right).
\end{eqnarray}
The empirical covariance-based dissimilarity measure between a sample path $\mathbf x_i$ and a process $X^{(j)}$ ($i,j\in\{1,2\}$) is defined by
\begin{eqnarray}
\label{dxx_1}
&&\widehat{d^*}(\mathbf x_{i},X^{(j)}):=\sum_{m= 1}^{m_n} \sum_{l= 1}^{n-m+1} w_m w_l\nonumber\\
&&\times\mathcal M\left( \left(\mu^*(X^{(i)}_{l\ldots n},m), \nu^*(X^{(i)}_{l\ldots n},m)\right),\left(\mathbb E\left(X_{l\ldots l+m-1}^{(j)}\right), \mathbb Cov\left(X_{l\ldots l+m-1}^{(j)}\right)\right)\right).\nonumber\\
\end{eqnarray}
\end{definition}
Unlike the dissimilarity measure $d^*$ which describes some distance between stochastic processes, the empirical covriance-based dissimilarity measure is some distance between two sample paths (finite-length vectors). We will show in the forthcoming Lemma \ref{lemma1} that $\widehat{d^*}$ is a consistent estimator of $d^*$.

Two observed sample paths possibly have distinct lengths $n_1, n_2$, therefore in  (\ref{dxx}) we consider computing the distances between their subsequences of length $n=\min\{n_1,n_2\}$. In practice we usually take $m_n=\lfloor \log n\rfloor$, the floor number of $\log n$. 

It is easy to verify that both $d^*$ and $\widehat{d^*}$ satisfy the triangle inequalities, thanks to the fact that both the Euclidean distance and  $\rho^*$ satisfy the triangle inequalities. More precisely, the following holds.
\begin{remark}
\label{rmk_1}
Thanks to (\ref{def:rho}) and the definitions of $d^*$ (see (\ref{def:d*})) and $\widehat{d^*}$ (see (\ref{dxx})), we see that the triangle inequality holds for the covariance-based dissimilarity measure $d^*$, as well as for its empirical estimate $\widehat{d^*}$. Therefore for arbitrary processes $X^{(i)},~i = 1,2,3$ and arbitrary finite-length sample paths $\mathbf x_{i},~i=1,2,3$, we have
\begin{eqnarray*}
&&d^*\big(X^{(1)},X^{(2)}\big)  \leq d^*\big(X^{(1)},X^{(3)}\big) + d^*\big(X^{(2)},X^{(3)}\big),
\\
&&\widehat {d^*}(\mathbf{x}_{1},\mathbf{x}_{2})  \leq \widehat{d^*}(\mathbf{x}_{1},\mathbf{x}_{3}) + \widehat{d^*}(\mathbf{x}_{2},\mathbf{x}_{3}),
\\
&&\widehat{d^*}\big(\mathbf{x}_{1},X^{(1)}\big) \leq \widehat{d^*}\big(\mathbf{x}_{1},X^{(2)}\big) + d^*\big(X^{(1)},X^{(2)}\big).
\end{eqnarray*}
\end{remark}
Remark \ref{rmk_1} together with the fact that the processes are weakly ergodic, leads to Lemma \ref{lemma1} below, which is the key to demonstrate that our clustering algorithms in the forthcoming section are asymptotically consistent.
\begin{lemma} \label{lemma1}
Given two paths
$$
\mathbf{x_1}=\left(X_1^{(1)},\ldots,X_{n_1}^{(1)}\right) \quad \mbox{and} \quad \mathbf{x_2}=\left(X_1^{(2)},\ldots,X_{n_2}^{(2)}\right),
$$
sampled from the wide-sense stationary ergodic processes $X^{(1)}$ and $X^{(2)}$ respectively, we have
\begin{align}
\label{limdxx}
\mathbb P\left(\lim_{n_1,n_2 \rightarrow \infty} \widehat{d^*}\left(\mathbf{x}_{1},\mathbf{x}_{2}\right) = d^*\left(X^{(1)},X^{(2)}\right)\right)=1
\end{align}
and
\begin{align}
\label{limdxx1}
\mathbb P\left(\lim_{n_i \rightarrow \infty} \widehat{d^*}\left(\mathbf{x}_{i},X^{(j)}\right) = d^*\left(X^{(1)},X^{(2)}\right)\right)=1,~\mbox{for}~i,j\in\{1,2\},~i\neq j.
\end{align}
\end{lemma}
\begin{proof}
We take $n=\min\{n_1,n_2\}$. To show (\ref{limdxx}) holds it suffices to prove that for arbitrary $\varepsilon>0$, there is an integer $N>0$ such that for any $n\ge N$, with probability 1,
$$
\left|\widehat{d^*}\left(\mathbf{x}_{1},\mathbf{x}_{2}\right) - d^*(X^{(1)},X^{(2)})\right|<\varepsilon.
$$
Define the sets of indexes
$$
S_1(n)=\left\{(m,l)\in\mathbb N^2:~m\le m_n,~l\le n-m+1\right\}~\mbox{and}~S_2(n)=\mathbb N^2\backslash S_1(n).
$$
To be more convenient we also denote by
\begin{equation}
\label{Vlm}
V\Big(X_{l\ldots l+m-1}^{(j)}\Big):=\left(\mathbb E\left(X_{l\ldots l+m-1}^{(j)}\right),\mathbb Cov\left(X_{l\ldots l+m-1}^{(j)}\right)\right)
\end{equation}
and
\begin{equation}
\label{est_Vlm}
\widehat{V}\Big(X_{l\ldots n}^{(j)},m\Big):=\left(\mu^*\left(X_{l\ldots n}^{(j)},m\right),\nu^*\left(X_{l\ldots n}^{(j)},m\right)\right),
\end{equation}
for $(m,l)\in\mathbb N^2$ and $j=1,2$. By using the definitions of $d^*$ (see (\ref{def:d*})), of $\widehat{d^*}$ (see (\ref{dxx})) and the triangle inequality $$
\left|\sum\limits_{i\in I}a_i\right|\le \sum\limits_{i\in I}|a_i|,~\mbox{for any indexes set $I$ and any real numbers $a_i$'s},
$$ we obtain
\begin{eqnarray}
\label{hatdd1}
&&\left| \widehat{d^*}(\mathbf{x}_{1}, \mathbf{x}_{2}) - d^*\big(X^{(1)}, X^{(2)}\big)\right|\nonumber\\
&&= \biggl| \sum_{(m,l)\in S_1(n)} w_m w_l \left(\mathcal M\left ( \widehat{V}(X_{l\ldots n}^{(1)},m),\widehat{V}(X_{l\ldots n}^{(2)},m)\right )\right.\nonumber\\
&&\hspace{1cm}- \sum_{(m,l)\in S_1(n)\cup S_2(n)} w_m w_l\mathcal M\left (V(X_{l\ldots l+m-1}^{(1)}), V(X_{l\ldots l+m-1}^{(2)})\right ) \biggl| \nonumber\\
&&\le \biggl| \sum_{(m,l)\in S_1(n)} w_m w_l \left(\mathcal M\left ( \widehat{V}(X_{l\ldots n}^{(1)},m),\widehat{V}(X_{l\ldots n}^{(2)},m)\right )\right.\nonumber\\
&&\hspace{1cm}\left.- \mathcal M\left (V(X_{l\ldots l+m-1}^{(1)}), V(X_{l\ldots l+m-1}^{(2)})\right )\right) \biggl| \nonumber\\
&&\hspace{2cm}+\sum_{(m,l)\in S_2(n)}w_mw_l\mathcal M\left (V(X_{l\ldots l+m-1}^{(1)}), V(X_{l\ldots l+m-1}^{(2)})\right ) \nonumber\\
&&\le  \sum_{(m,l)\in S_1(n)} w_m w_l \biggl|\mathcal M\left ( \widehat{V}(X_{l\ldots n}^{(1)},m),\widehat{V}(X_{l\ldots n}^{(2)},m)\right )\nonumber\\
&&\hspace{1cm}- \mathcal M\left (V(X_{l\ldots l+m-1}^{(1)}), V(X_{l\ldots l+m-1}^{(2)})\right )\biggl| \nonumber\\
&&\hspace{2cm}+\sum_{(m,l)\in S_2(n)}w_mw_l\mathcal M\left (V(X_{l\ldots l+m-1}^{(1)}), V(X_{l\ldots l+m-1}^{(2)})\right).
\end{eqnarray}
Next note that the metric $\mathcal M$ satisfies the following triangle inequality:
\begin{eqnarray}
\label{rhorho}
&& \biggl|\mathcal M\left ( \widehat{V}(X_{l\ldots n}^{(1)},m),\widehat{V}(X_{l\ldots n}^{(2)},m)\right )- \mathcal M\left (V(X_{l\ldots l+m-1}^{(1)}), V(X_{l\ldots l+m-1}^{(2)})\right )\biggl|\nonumber\\
&&\leq \mathcal M\left ( \widehat{V}(X_{l\ldots n}^{(1)},m),{V}(X_{l\ldots l+m-1}^{(1)})\right )+ \mathcal M\left (\widehat V(X_{l\ldots n}^{(2)},m), V(X_{l\ldots l+m-1}^{(2)})\right).
\end{eqnarray}
It follows from (\ref{hatdd1}) and (\ref{rhorho}) that
\begin{eqnarray}
\label{hatdd2}
&&\left| \widehat{d^*}(\mathbf{x}_{1}, \mathbf{x}_{2}) - d^*\big(X^{(1)}, X^{(2)}\big)\right|\nonumber\\
&&\le  \sum_{(m,l)\in S_1(n)} w_m w_l \biggl(\mathcal M\left ( \widehat{V}(X_{l\ldots n}^{(1)},m),V(X_{l\ldots l+m-1}^{(1)})\right )\nonumber\\
&&\hspace{1cm}+ \mathcal M\left (\widehat V(X_{l\ldots n}^{(2)},m), V(X_{l\ldots l+m-1}^{(2)})\right )\biggl) \nonumber\\
&&\hspace{2cm}+\sum_{(m,l)\in S_2(n)}w_mw_l\mathcal M\left (V(X_{l\ldots l+m-1}^{(1)}), V(X_{l\ldots l+m-1}^{(2)})\right ).
\end{eqnarray}
Next we show that the right-hand side of (\ref{hatdd2}) converges to $0$ as $n\to\infty$. First observe that the weights $\{w_m\}_{m\ge1}$ have been chosen such that
\begin{equation}
\label{assump_weights}
\sum_{m,l = 1}^{\infty} w_m w_l \mathcal M\left(V(X_{l\ldots l+m-1}^{(1)}),V(X_{l\ldots l+m-1}^{(2)})\right)<+\infty.
\end{equation}
Then for arbitrary fixed $\varepsilon>0$, we can find an index $J$ such that for $n\ge J$,
\begin{equation}
\label{bound_omega}
\sum_{(m,l)\in S_2(n)} w_mw_l\mathcal M\left(V(X_{l\ldots l+m-1}^{(1)}),V(X_{l\ldots l+m-1}^{(2)})\right)\leq\frac{\varepsilon}{3}.
\end{equation}
Next, the weak ergodicity of the processes $X^{(1)}$ and $X^{(2)}$ implies that: for each $(m,l)\in\mathbb N^2$, $\widehat V(X_{l\ldots n}^{(j)},m)$ ($j=1,2$) is a strongly consistent estimator of  $V(X_{l\ldots l+m-1}^{(j)})$, under the metric $\mathcal M$, i.e., with probability 1,
\begin{equation}
\label{limvV}
\lim_{n\to\infty}\mathcal M\left(\widehat V(X_{l\ldots n}^{(j)},m),~ V(X_{l\ldots l+m-1}^{(j)})\right)=0.
\end{equation}
Thanks to (\ref{limvV}), for any $(m,l)\in S_1(J)$, there exists some $N_{m,l}$ (which depends on $m,l$) such that for all $n\geq N_{m,l}$, we have, with probability $1$, 
\begin{align}
\label{bound_omega'}
\mathcal M\left ( \widehat V(X^{(j)}_{l\ldots n},m),~V(X_{l\ldots l+m-1}^{(j)})\right )\leq \frac{\varepsilon}{3w_m w_l\#S_1(J)},~\mbox{for}~j = 1, 2,
\end{align}
where $\# A$ denotes the number of elements included in the set $A$. 
Denote by $N_J=\max\limits_{(m,l)\in S_1(J)}N_{m,l}$. Then observe that, for $n\ge \max\{N_J,J\}$,
\begin{eqnarray}
\label{n<J}
&&\sum_{(m,l)\in S_2(n)}w_mw_l\mathcal M\left (V(X_{l\ldots l+m-1}^{(1)}), V(X_{l\ldots l+m-1}^{(2)})\right )\nonumber\\
&&\le \sum_{(m,l)\in S_2(J)}w_mw_l\mathcal M\left (V(X_{l\ldots l+m-1}^{(1)}), V(X_{l\ldots l+m-1}^{(2)})\right ).
\end{eqnarray}
It results from (\ref{hatdd2}), (\ref{n<J}), (\ref{bound_omega'}) and (\ref{bound_omega}) that, for $n\ge \max\{N_J,J\}$,
\begin{eqnarray*}
&&\left| \widehat{d^*}(\mathbf{x}_{1}, \mathbf{x}_{2}) - d^*\big(X^{(1)}, X^{(2)}\big)\right|\nonumber\\
&&\leq  \sum_{(m,l)\in S_1(n)} w_m w_l  \mathcal M\left(\widehat V(X^{(1)}_{l\ldots n},m),V(X_{l\ldots l+m-1}^{(1)})\right )\nonumber\\
&&\hspace{1cm}+\sum_{(m,l)\in S_1(n)} w_m w_l \mathcal M\left ( \widehat V(X^{(2)}_{l\ldots n},m),V(X_{l\ldots l+m-1}^{(2)})\right)\nonumber
\\
&&\hspace{2cm}+\sum_{(m,l)\in S_2(J)}w_mw_l\mathcal M\left (V(X_{l\ldots l+m-1}^{(1)}), V(X_{l\ldots l+m-1}^{(2)})\right )\nonumber\\
&&\le \frac{\varepsilon}{3}+\frac{\varepsilon}{3}+\frac{\varepsilon}{3}=\varepsilon,
\end{eqnarray*}
which proves (\ref{limdxx}). The statement (\ref{limdxx1}) can be proved analogously.
\end{proof}

\section{Asymptotically Consistent Clustering Algorithms} \label{sec::algo_consist}
\subsection{Offline and Online Algorithms}
In this section we introduce the asymptotically consistent algorithms for clustering offline and online datasets respectively. We explain how the two algorithms work, and prove that both algorithms are asymptotically consistent. It is worth noting that the asymptotic consistency of our algorithms relies on the assumption that  the number of clusters $\kappa$ is priorly known. The case for $\kappa$ being unknown has been studied in \cite{khaleghi2016} in the problem of clustering strictly stationary ergodic processes. However in the setting of wide-sense stationary ergodic processes, this problem remains open.

Algorithm \ref{algo::offline_known_k} below presents the pseudo-code for clustering offline datasets. It is a centroid-based clustering approach. One of its main features is that the farthest 2-point initialization applies. The algorithm selects the first two cluster centers by picking the two \enquote{farthest} observations among all observations (Lines 1 - 3), under the empirical dissimilarity measure $\widehat{d^*}$. Then each next cluster center is chosen to be the observation farthest to all the previously assigned cluster centers (Lines 4 - 6). Finally the algorithm assigns each remaining observation to its nearest cluster (Lines 7-11).\\

\begin{algorithm}[H] \label{algo::offline_known_k}
\caption{Offline clustering, with known $\kappa$}
\LinesNumbered 
\KwIn{\textbf{\textit{sample paths}} $S= \left \{ \mathbf{x}_1, \ldots, \mathbf{x}_N \right \}$; \textbf{\textit{number}} $\kappa$ \textbf{\textit{of clusters}};  \textbf{\textit{weights} $w_j$, $j=1,\ldots,N(t)$.}}

$(c_1,c_2) \longleftarrow \argmax\limits_{(i,j)\in\{1,\ldots,N\}^2, i<j}\widehat{d^*}(\mathbf x_i,\mathbf x_j)$\;
$C_1 \longleftarrow \left \{ c_1 \right \}$\; $C_2\longleftarrow\{c_2\}$\;
\For{$k = 3,\ldots,\kappa$}{
$c_k \longleftarrow \displaystyle\argmax_{i=1,\ldots,N}\displaystyle\min_{j = 1,\ldots,k-1} \widehat{d^*}(\mathbf{x}_i, \mathbf{x}_{c_j})$\;
}
\textbf{\textit{Assign each remaining point to its nearest cluster center}:}

\For{$i = 1,\ldots,N$}{
$k \longleftarrow \argmin\limits_{k\in\{1,\ldots,\kappa\}}\left\{\widehat{d^*}(\mathbf{x}_i, \mathbf{x}_j):~j \in C_k\right\}$;\\
$C_k \longleftarrow C_k \cup \left \{ i\right \}$\;
}
\KwOut{The $\kappa$ clusters $\{C_1,C_2, \ldots, C_\kappa\}$.}
\end{algorithm}
We point out that Algorithm \ref{algo::offline_known_k} is different from Algorithm 1 in \cite{khaleghi2016} at two points:
\begin{enumerate}
\item As mentioned previously, our algorithm relies on the covariance-based dissimilarity $\widehat{d^*}$, in lieu of the process distributional distances.
\item Our algorithm suggests $2$-point initialization, while Algorithm 1 in \cite{khaleghi2016} randomly picks $1$-point as the first cluster center. The latter initialization was proposed for use with $k$-means clustering by \cite{katsavounidis1994}. Algorithm 1 in \cite{khaleghi2016} requires $\kappa N$ distance calculations, while our algorithm requires $N(N-1)/2$ distances calculations. It is very important to point out that, to reduce the computational complexity cost of our algorithm, it is fine to replace our $2$-point initialization with the one in \cite{khaleghi2016}. However there are two reasons based on which we recommend using our approach of initialization:
\begin{description}
\item[\textbf{Reason 1}] In the forthcoming Section \ref{sub:simulation}, our empirical comparison to \cite{khaleghi2016} shows that the $2$-point initialization turns out to be more accurate in clustering than the $1$-point initialization.
\item[\textbf{Reason 2}] Concerning the complexity cost, we have the following loss and earn: on one hand, the $2$-point initialization requires more steps of calculations than the $1$-point initialization; on the other hand, in our covariance-based dissimilarity measure $\widehat{d^*}$ defined in (\ref{dxx}), the matrices distance $\rho^*$ requires $m_n^2$ computations of Euclidean distances, while the distance  $
\sum_{B\in B^{m,l}}|\nu(\mathbf x_1,B)-\nu(\mathbf x_2,B)|$ 
given in (\ref{hatd1}) requires at least $n_1+n_2-2m_n+2$ computations of Euclidean distances (see Eq. (33) in \cite{khaleghi2016}). Note that we take $m_n=\lfloor\log n\rfloor$ ($\lfloor\cdot\rfloor$ denotes the floor integer number) though this framework. Therefore the computational complexity of the covariance-based dissimilarity $\widehat{d^*}$ makes the overall complexity of Algorithm \ref{algo::offline_known_k} quite competitive to the algorithm in \cite{khaleghi2016}, especially when the paths lengths $n_i$, $i=1,\ldots,n$ are relatively large, or when the database of all distance values are at hand.
\end{description}
\end{enumerate}
Next we present the clustering algorithm for online setting. As mentioned in \cite{khaleghi2016}, one regards recently-observed paths as unreliable observations, for which sufficient information has not yet been collected, and for which the estimators of the covariance-based dissimilarity measures are not accurate enough. Consequently, farthest-point initialization would not work in this case; and clustering on all available data results in not only mis-clustering unreliable paths, but also in clustering incorrectly those for which sufficient data are already available. The strategy is presented in Algorithm \ref{algo::online_known_k} below: clustering based on a weighted combination of several clusterings, each obtained by running the offline algorithm (Algorithm \ref{algo::offline_known_k}) on different portions of data.

More precisely, Algorithm \ref{algo::online_known_k} works as follows. Suppose the number of clusters $\kappa$ is known. At time $t$, a sample $S(t)$ is observed (Lines 1 - 2), the algorithm iterates over $j= \kappa,\ldots,N(t)$ where at each iteration Algorithm \ref{algo::offline_known_k} is utilized to cluster the first $j$ paths in $S(t)$ into $\kappa$ clusters (Lines 6 - 7). For each cluster its center is selected as the observation having the \textit{smallest} index among that cluster, and their indexes are ordered increasingly (Line 8). The minimum inter-cluster distance $\gamma_j$ (see \cite{cesa2006}) is calculated as the minimum distance $\widehat{d^*}$ between the $\kappa$ cluster centers obtained at iteration $j$ (Line 9).  Finally, every observation in $S(t)$ is assigned to the nearest cluster, based on the weighted combination of the distances between this observation and the candidate cluster centers obtained at each iteration on $j$ (Lines 14 - 17).

\begin{algorithm}[H] \label{algo::online_known_k}
\caption{Online clustering, with known $\kappa$}
\LinesNumbered 
\KwIn{\textbf{\textit{sample paths}} $\Big\{S(t)=\{\mathbf x_1^t,\ldots,\mathbf x_{N(t)}^t\}\Big\}_t$; \textbf{\textit{number of clusters}} $\kappa$; \textbf{\textit{weights} $\beta(j)$, $j=1,\ldots,N(t)$}.}

\For{$t = 1,\ldots,\infty$}{
\textbf{\textit{Obtain new sequences:} $S(t) \longleftarrow \Big\{\mathbf{x}_1^t,\dots,\mathbf{x}_{N(t)}^t\Big \}$}\;
\textbf{\textit{Initialize the normalization factor}: $\eta \longleftarrow 0$}\;
\textbf{\textit{Initialize the final clusters}: $C_k(t) \longleftarrow \emptyset,~k = 1,\ldots,\kappa$}\;
\textbf{\textit{Generate }$N(t) -\kappa + 1$ \textit{candidate cluster centers}:}

\For{$j = \kappa,\ldots,N(t)$}{
$\big \{C_1^j,\ldots, C_{\kappa}^j\big \} \longleftarrow\mbox{\textbf{Alg1}}\big( \big \{\mathbf{x}_1^t,\ldots, \mathbf{x}_j^t \big \}, \kappa \big)$\;
$(c_1^j,\ldots,c_\kappa^j) \longleftarrow \mbox{sort}(\min \big \{ i\in C_k^j \big \}, k = 1,\ldots,\kappa)$\;
$\gamma_j \longleftarrow \min\limits_{k,k'\in \{1,\ldots,\kappa\},k\neq k'} \widehat{d^*}\big(\mathbf{x}_{c_k^j}^t,\mathbf{x}_{c_{k'}^j}^t\big)$\;
$w_j \longleftarrow \beta(j)$\;
$\eta \longleftarrow \eta+w_j\gamma_j$\;}
\textbf{\textit{Assign each point to a cluster}:}

\For{$i = 1,\ldots,N(t)$}{
$k \longleftarrow \argmin\limits_{k'\in \{1,\ldots,\kappa\}}\frac{1}{\eta} \sum\limits_{j=\kappa}^{N(t)}w_j \gamma_j \widehat{d^*}\big(\mathbf{x}_{i}^t,\mathbf{x}_{c_{k'}^j}^t\big)$\;
$C_k(t) \longleftarrow C_k(t)\cup \left\{ i \right\}$\;
}}
\KwOut{ \textit{The} $\kappa$ \textit{clusters} $\left \{ C_1(t), \dots, C_{\kappa}(t) \right \}$, $t=1,2,\ldots,\infty$.}
\end{algorithm}
In Algorithm \ref{algo::online_known_k}, $\beta(j)$ denotes a function indexed by $j$, which is the value chosen for the weight $w_j$. Remark that for online setting, our algorithm requires the same number of distance calculations as in Algorithm 2 in \cite{khaleghi2016}. They are both bounded by $\mathcal O(N(t)^2)$. Using 2-point initialization, our Algorithm \ref{algo::online_known_k} then takes advantage in the overall computational complexity cost. Finally we note that both Algorithm \ref{algo::offline_known_k} and Algorithm \ref{algo::online_known_k} require $\kappa\ge2$. When $\kappa$ is known, this restriction is not a practical issue. 

\subsection{Consistency and Computational Complexity of the Algorithms}
In this section we prove the asymptotic consistency of Algorithms \ref{algo::offline_known_k} and \ref{algo::online_known_k}.  They are stated in the 2 theorems below.
\begin{theorem}
\label{thm:offline}
Algorithm \ref{algo::offline_known_k} is strongly asymptotically consistent (in the offline sense), provided that the true number $\kappa$ of clusters is known, and each sequence $\mathbf{x}_i,~i = 1,\ldots,N$ is sampled from some wide-sense stationary ergodic process.
\end{theorem}

\begin{proof}
Similar to the idea used in the proof of Theorem 11 in \cite{khaleghi2016}, to prove the consistency statement we will need Lemma \ref{lemma1} to show that if the sample paths in $S$ are long enough, the sample paths that are generated by the same process covariance structure are \enquote{closer} to each other than to the rest. Therefore, the sample paths chosen as cluster centers are each generated by a different covariance structure, and since the algorithm assigns the rest to the closest clusters, the statement follows. More formally, let $n_{\min}$ denote the shortest path length in $S$:
\begin{align*}
n_{\min}:= \min\left\{n_i:~i=1,\ldots,N\right\}.
\end{align*}
Denote by $\delta_{\min}$ the minimum non-zero covariance-based dissimilarity measure between any 2 covariance structures:
\begin{equation}
\label{delta:kk'}
\delta_{\min}:= \min\left\{d^*\left(X^{(k)},X^{(k')}\right):~k,k'\in\{1,\ldots,\kappa\},~k\neq k'\right\}.
\end{equation}
Fix $\varepsilon \in (0, \delta_{\min}/4)$. Since there are a finite number $N$ of observations, by Lemma \ref{lemma1} there is $n_0$ such that for $n_{\min}\ge n_0$ we have
\begin{align}
\label{d:epsilon}
\max_{\substack{l \in \{1,\ldots,\kappa\} \\i \in G_l \cap \left \{ 1,\ldots,N \right \}}} \widehat{d^*}\left(\mathbf{x}_i, X^{(l)}\right)\leq\varepsilon,
\end{align}
where $G_l,~l = 1,\ldots,\kappa$ denote the covariance structure ground-truth partitions given by Definition \ref{ground-truth}. 

On one hand, by using (\ref{d:epsilon}), the triangle inequality (see Remark \ref{rmk_1}) and the fact that 
$$
\max_{i\in I}(a_i+b_i)\le \max_{i\in I}a_i+\max_{i\in I}b_i
$$
for any indexes set $I$ and any real numbers $a_i$'s and $b_i$'s, we obtain
\begin{eqnarray}
\label{upperbound}
&&\max_{\substack{l \in \{1,\ldots,\kappa\} \\i,j \in G_l \cap \left \{ 1,\ldots,N \right \}}} \widehat{d^*}\left(\mathbf{x}_i, \mathbf{x}_j\right)\nonumber\\
&&\leq \max_{\substack{ l \in \{1,\ldots,\kappa\} \\i,j \in G_l \cap \left \{ 1,\ldots,N \right \}}} \widehat{d^*}\left(\mathbf{x}_i, X^{(l)}\right)+\max_{\substack{l \in \{1,\ldots,\kappa\} \\i,j \in G_l \cap \left \{ 1,\ldots,N \right \}}} \widehat{d^*}\left(\mathbf{x}_j, X^{(l)}\right)\nonumber\\
&&= \max_{\substack{ l \in \{1,\ldots,\kappa\} \\i \in G_l \cap \left \{ 1,\ldots,N \right \}}} \widehat{d^*}\left(\mathbf{x}_i, X^{(l)}\right)+\max_{\substack{l \in \{1,\ldots,\kappa\} \\j \in G_l \cap \left \{ 1,\ldots,N \right \}}} \widehat{d^*}\left(\mathbf{x}_j, X^{(l)}\right)\nonumber\\
&&\le 2 \varepsilon< \frac{\delta_{\min}}{2}.
\end{eqnarray}
On the other hand, by using the triangle inequality (see Remark \ref{rmk_1}), (\ref{delta:kk'}) and (\ref{d:epsilon}), we have for $n_{\min}\ge n_0$,
\begin{eqnarray}
\label{lowerbound}
&&\min_{\substack{ k,k'\in\{1,\ldots,\kappa\},k\neq k'\\ i \in G_k \cap \left \{ 1,\ldots,N \right \} \\j \in G_{k'} \cap \left \{ 1,\ldots,N \right \}}} \widehat{d^*}(\mathbf{x}_i, \mathbf{x}_j) \nonumber\\
&&\geq \min_{\substack{k,k'\in\{1,\ldots,\kappa\},k\neq k'\\ i \in G_k \cap \left \{ 1,\ldots,N \right \} \\j \in G_{k'} \cap \left \{ 1,\ldots,N \right \}}}\left\{ d^*\left(X^{(k)}, X^{(k')}\right) - \widehat{d^*}\left(\mathbf{x}_i, X^{(k)}\right)- \widehat{d^*}\left(\mathbf{x}_j, X^{(k')}\right)\right\}  \nonumber\\
&&\geq \delta_{\min}-2\varepsilon> \frac{\delta_{\min}}{2}.
\end{eqnarray}
In words, (\ref{upperbound}) together with (\ref{lowerbound}) indicates that the sample paths in $S$ that are generated by the same covariance structure are closer to each other than to the rest of sample paths. Then by (\ref{upperbound}) and (\ref{lowerbound}), for $n_{\min}\ge n_0$, we necessarily have each sample path should be \enquote{close} enough to its cluster center, i.e.,
\begin{equation}
\label{center}
\max_{i = 1,\ldots,N} \min_{k = 1,\ldots,\kappa-1} \widehat{d^*} (\mathbf{x}_i, \mathbf{x}_{c_k})>\frac{\delta_{\min}}{2},
\end{equation}
where the $\kappa$ cluster centers' indexes $c_1,\ldots,c_\kappa$ are given by Algorithm \ref{algo::offline_known_k} as
$$
(c_1,c_2) := \argmax_{i,j = 1,\ldots,N,~i<j}\widehat{d^*} (\mathbf{x}_i, \mathbf{x}_{j}),
$$
and
$$
c_k :=\argmax_{i = 1,\ldots,N} \displaystyle\min_{j = 1,\ldots,k-1} \widehat{d^*} (\mathbf{x}_i, \mathbf{x}_{c_j}),~k = 3,\ldots,\kappa.
$$
Hence, the indexes $c_1,\ldots, c_{\kappa}$ will be chosen to index the sample paths generated by different process covariance structures. Then by (\ref{upperbound}) and (\ref{lowerbound}), each remaining sample path will be assigned to the cluster center corresponding to the sample path generated by the same process covariance structure. Finally Theorem \ref{thm:offline} results from (\ref{upperbound}), (\ref{lowerbound}) and (\ref{center}).
\end{proof}

\begin{theorem}
\label{thm:online}
Algorithm \ref{algo::online_known_k} is strongly asymptotically consistent (in the online sense), provided the true number of clusters $\kappa$ is known, and each sequence $\mathbf{x}_i, i \in \mathbb{N}$ is sampled from some wide-sense stationary ergodic process.
\end{theorem}

\begin{proof}
The idea of the proof is similar to that of Theorem 12 in \cite{khaleghi2016}. The main  differences between the 2 proofs are made by the fact that our covariance-based dissimilarity measure $\widehat{d^*}$ is not bounded by some constant. Although it is not mentioned in the pseudo-code Algorithm \ref{algo::online_known_k}, the notations $\gamma_j$'s and $\eta$ are dependent of $t$, therefore we denote $\gamma_j^t:=\gamma_j$ and $\eta^t:=\eta$ through this proof.  
In the first step, by using the triangle inequality we can show that for any $t>0$, any $N\in\mathbb N$,
\begin{eqnarray}
\label{bound_dt1}
&&\sup_{\substack{j\in\{1,\ldots,N\}\\ k\in\{1,\ldots,\kappa\}}}\widehat{d^*} \left(\mathbf{x}_{j}^t ,X^{(k)} \right)\le \sup_{\substack{j\in\{1,\ldots,N\}\\ k\in\{1,\ldots,\kappa\}}}\left(d^*\left(X^{(k)}, X^{(k_j')}\right) +\widehat{d^*}\left(\mathbf{x}_{j}^t, X^{(k_j')}\right) \right)\nonumber\\
&&\le \sup_{\substack{j\in\{1,\ldots,N\}\\ k\in\{1,\ldots,\kappa\}}}d^*\left(X^{(k)}, X^{(k_j')}\right) +\sup_{\substack{j\in\{1,\ldots,N\}\\ k\in\{1,\ldots,\kappa\}}}\widehat{d^*}\left(\mathbf{x}_{j}^t, X^{(k_j')}\right)\nonumber\\
&&= \sup_{\substack{j\in\{1,\ldots,N\}\\ k\in\{1,\ldots,\kappa\}}}d^*\left(X^{(k)}, X^{(k_j')}\right) +\sup_{j\in\{1,\ldots,N\}}\widehat{d^*}\left(\mathbf{x}_{j}^t, X^{(k_j')}\right),
\end{eqnarray}
where for each $j$, $k_j'$ is chosen such that $\mathbf x_j^t$ is sampled from the process covariance structure $X^{(k_j')}$. On one hand,
let
\begin{equation}
\label{deltamax:kk'}
\delta_{\max}:= \max\left\{d^*\left(X^{(k)},X^{(k')}\right):~k,k'\in\{1,\ldots,\kappa\},~k\neq k'\right\},
\end{equation}
then the first term on the right-hand side of (\ref{bound_dt1}) can be bounded by the constant $\delta_{\max}$, which neither depends on $t$ nor on $N$:
\begin{equation}
    \label{bound_right_1}
    \sup_{\substack{j\in\{1,\ldots,N\}\\ k\in\{1,\ldots,\kappa\}}}d^*\left(X^{(k)}, X^{(k_j')}\right) \le \delta_{\max}.
\end{equation}
On the other hand, since $\mathbf x_j^t$ is sampled from $X^{(k_j')}$, by using the weak ergodicity (see Lemma \ref{lemma1}), for $j=1,\ldots,N$, with probability $1$,
$$
\lim_{t\to\infty}\widehat{d^*}\left(\mathbf{x}_{j}^t, X^{(k_j')}\right)=0.
$$
This together with the fact that a convergent sequence is also bounded, leads to, for each $j\in\{1,\ldots, N\}$, there is $b_j$ (not depending on $t$) such that
$$
\widehat{d^*}\left(\mathbf{x}_{j}^t, X^{(k_j')}\right)\le b_j,~\mbox{for all $t\ge 0$}.
$$
Therefore the second term on the right-hand side of (\ref{bound_dt1}) can be bounded as:
\begin{equation}
    \label{bound_right_2}
\sup_{j\in\{1,\ldots,N\}}\widehat{d^*}\left(\mathbf{x}_{j}^t, X^{(k_j')}\right)\le\max\{b_1,\ldots,b_N\}.
\end{equation}
Let 
\begin{equation}
    \label{bound_right_3}
B(N):=\delta_{\max}+\max\{b_1,\ldots,b_N\}.
\end{equation}
It is important to point out that $B(N)$ depends only on $N$ but not on $t$. It follows from (\ref{bound_dt1}),  (\ref{bound_right_1}), (\ref{bound_right_2}) and (\ref{bound_right_3}) that
\begin{equation}
\label{bound_dt}
\sup_{\substack{j\in\{1,\ldots,N\}\\ k\in\{1,\ldots,\kappa\}}}\widehat{d^*} \left(\mathbf{x}_{j}^t ,X^{(k)} \right)\le B(N).
\end{equation}
Let $\delta_{\min}$ be the one given in (\ref{delta:kk'}). Fix $\varepsilon \in (0, \delta_{\min}/4)$. By using (\ref{sum_weights}), we can choose some $J>0$ so that 
\begin{equation}
\label{bound_wJ}
\sum_{j = J+1}^\infty w_j\le\varepsilon.
\end{equation}
Recall that in online setting, the $i$th sample path's length $n_i(t)$ grows with time, for each $i$. Therefore, by the wide-sense ergodicity (see Lemma \ref{lemma1}), for every $j \in  \{1,\ldots,J\}$ there exists some $T_1(j)>0$ such that for all $t \geq T_1(j)$ we have
\begin{align}
\label{d:upperbound}
 \max_{\substack{k \in \{1,\ldots,\kappa\} \\ i \in G_k \cap \left \{ 1,\ldots,j \right \}} } \widehat{d^*}\left(\mathbf{x}_i^t, X^{(k)}\right) \le\varepsilon.
\end{align}
For $k = 1,\ldots,\kappa$, define $s_k(N(t))$ to be the index of the first path in $S(t)$ sampled from the covariance structure $X^{(k)}$, i.e.,
\begin{align}
\label{sk}
s_k(N(t)) := \min \left \{ i \in G_k \cap \{1,\ldots,N(t)\} \right \}.
\end{align}
Note that $s_k(N(t))$ depends only on $N(t)$. 
Then denote
\begin{align}
\label{m_t}
m(N(t)) := \max_{k \in \{1,\ldots,\kappa\}} s_k(N(t)).
\end{align}
By Theorem \ref{thm:offline} for every $j \in \{m(N(t)),\ldots,J\}$ there exists some $T_2(j)$ such that $\mbox{Alg1}(S(t)|_j, \kappa)$ is asymptotically consistent for all $t \geq T_2(j)$, where $S(t)|_j = \left \{ \mathbf{x}_1^t, \ldots, \mathbf{x}_j^t \right \}$ denotes the subset of $S(t)$ consisting of the first $j$ sample paths.  Let
\begin{align*}
T:= \max_{\substack{i=1,2\\ j \in \{1,\ldots,J\}}} T_i(j).
\end{align*}
Recall that, by the definition of $m(N(t))$ in (\ref{m_t}), $S(t)|_{m(N(t))}$ contains sample paths from all $\kappa$ distinct covariance structures. Therefore, similar to obtaining (\ref{lowerbound}), for all $t \geq T$, we use the triangle inequality, (\ref{delta:kk'}) and (\ref{d:upperbound}) to obtain
\begin{eqnarray}
\label{lowerbound2}
&&\min_{\substack{ k,k'\in\{1,\ldots,\kappa\}\\ k\neq k'}} \widehat{d^*}\left(\mathbf{x}_{c_{k}^{m(N(t))}}^t , \mathbf{x}_{c_{k'}^{m(N(t))}}^t \right)\nonumber\\
&&\geq \min_{\substack{ k,k'\in\{1,\ldots,\kappa\}\\ k\neq k'}}\left(d^*\left(X^{(k)}, X^{(k')}\right) -\left(\widehat{d^*}\left(\mathbf{x}_{c_k^{m(N(t))}}^t, X^{(k)}\right) + \widehat{d^*}\left(\mathbf{x}_{c_{k'}^{m(N(t))}}^t, X^{(k')}\right)  \right)\right)\nonumber \\
&&\geq \delta_{\min} - 2\varepsilon \geq \frac{\delta_{\min}}{2}.
\end{eqnarray}
From Algorithm \ref{algo::online_known_k} (see Lines 9, 11) we see 
$$
\eta^t : = \sum_{j=1}^{N(t)}w_j\gamma_j^t,\quad\mbox{with}\quad
\gamma_j^t:=\min_{\substack{k,k'\in \{1,\ldots,\kappa\}\\ k\neq k'}} \widehat{d^*}\left(\mathbf{x}_{c_k^j}^t,\mathbf{x}_{c_{k'}^j}^t\right).
$$
Hence, by (\ref{lowerbound2}), for all $t\geq T$,
\begin{align}
\label{etabound}
\eta^t \geq \frac{w_{m(N(t))} \delta_{\min}}{2}.
\end{align}
For $j\in\{ J+1,\ldots,N(t)\}$, by the triangle inequality and (\ref{bound_dt}), we have for all $t\ge T$,
\begin{eqnarray}
\label{upperbound1}
&&\gamma_j^t=\min_{\substack{ k,k'\in\{1,\ldots,\kappa\}\\ k\neq k'}} \widehat{d^*}\left(\mathbf{x}_{c_{k}^{j}}^t , \mathbf{x}_{c_{k'}^{j}}^t \right)\nonumber\\
&&\le \min_{\substack{ k,k'\in\{1,\ldots,\kappa\}\\ k\neq k'}}\left(d^*\left(X^{(k)}, X^{(k')}\right) +\left(\widehat{d^*}\left(\mathbf{x}_{c_k^{j}}^t, X^{(k)}\right) + \widehat{d^*}\left(\mathbf{x}_{c_{k'}^{j}}^t, X^{(k')}\right)  \right)\right)\nonumber \\
&&\leq \delta_{\max} + 2B(N(t)).
\end{eqnarray}
Denote by
$$
M(N(t)):=\delta_{\max} + 2B(N(t)), 
$$
then (\ref{upperbound1}) can be interpreted as: for all $t\ge T$, 
\begin{equation}
\label{def:M}
\sup_{j\in\{J+1,\ldots,N(t)\}}\gamma_j^t\le M(N(t)).
\end{equation}
By (\ref{bound_dt}), (\ref{etabound}) and (\ref{def:M}), for every $k \in \{1,\ldots,\kappa\}$ we obtain
\begin{eqnarray}
\label{upper1}
&&\frac{1}{\eta^t}\sum_{j=1}^{N(t)} w_j \gamma_j^t\widehat{d^*}\left(\mathbf{x}_{c_k^j}^t ,X^{(k)} \right)\nonumber\\
&&=\frac{1}{\eta^t}\sum_{j=1}^{J} w_j \gamma_j^t\widehat{d^*}\left(\mathbf{x}_{c_k^j}^t ,X^{(k)} \right)+\frac{1}{\eta^t}\sum_{j=J+1}^{N(t)} w_j \gamma_j^t\widehat{d^*}\left(\mathbf{x}_{c_k^j}^t ,X^{(k)} \right)\nonumber\\
&&\leq \frac{1}{\eta^t}  \sum_{j=1}^{J} w_j \gamma_j^t \widehat{d^*} \left(\mathbf{x}_{c_k^j}^t ,X^{(k)} \right) + \frac{2B(N(t))M(N(t))}{w_{m(N(t))}\delta_{\min}}\sum_{j=J+1}^{N(t)}w_j\nonumber\\
&&= \frac{1}{\eta^t}  \sum_{j=1}^{m(N(t))-1} w_j \gamma_j^t \widehat{d^*} \left(\mathbf{x}_{c_k^j}^t ,X^{(k)} \right) +\frac{1}{\eta^t}  \sum_{j=m(N(t))}^{J} w_j \gamma_j^t \widehat{d^*} \left(\mathbf{x}_{c_k^j}^t ,X^{(k)} \right)\nonumber\\
&&\hspace{1cm}+ \frac{2B(N(t))M(N(t))\varepsilon}{w_{m(N(t))}\delta_{\min}}.
\end{eqnarray}
Next we provide upper bounds of the first 2 items in the right-hand side of (\ref{upper1}). On one hand, by the definition of $m(N(t)$, the sample paths in $S(t)|_j$ for $j = 1,\ldots,m(N(t)) - 1$ are generated by at most $\kappa -1$ out of the $\kappa$ process covariance structures. Therefore for each $j \in \{1,\ldots,m(N(t))-1\}$ there exists at least one pair of distinct cluster centers that are generated by the same process covariance structure. Consequently, by (\ref{d:upperbound}) and the definition of $\eta^t$, for all $t \geq T$ and $k\in\{1,\ldots,\kappa\}$,
\begin{equation}
\label{upper2}
\frac{1}{\eta^t} \sum_{j=1}^{m(N(t))-1} w_j \gamma_j^t \widehat{d^*}\left(\mathbf{x}_{c_k^j}^t ,X^{(k)}\right) \leq \frac{\varepsilon}{\eta^t} \sum_{j=1}^{m(N(t))-1} w_j \gamma_j^t \leq \varepsilon.
\end{equation}
On the other hand, since the clusters are ordered in the order of appearance of the distinct covariance structures, we have $\mathbf{x}_{c_l^j}^t = \mathbf{x}_{s_l(N(t))}^t$ for all $j = m,\ldots,J$ and $l = 1,\ldots,\kappa$, where the index $s_l(N(t))$ is defined in (\ref{sk}). Therefore, by (\ref{d:upperbound}) and the definition of $\eta^t$, for all $t \geq T$ and every $l=1,\ldots,\kappa$ we have
\begin{equation}
\label{upper3}
\frac{1}{\eta^t} \displaystyle \sum_{j=m(N(t))}^{J}w_j \gamma_j^t \widehat{d^*}\left(\mathbf{x}_{c_l^j}^t,X^{(l)}\right) =  \widehat{d^*} \left(\mathbf{x}_{s_l(N(t))}^t ,X^{(l)}\right) \frac{1}{\eta^t}\sum_{j=m(N(t))}^{J} w_j \gamma_j^t \leq \varepsilon.
\end{equation}
Combining (\ref{upper1}), (\ref{upper2}), (\ref{upper3}) and (\ref{d:upperbound}) we obtain, for $t\ge T$,
\begin{eqnarray}
\label{conv_proof}
\frac{1}{\eta^t} \displaystyle \sum_{j=1}^{N(t)} w_j \gamma_j^t \widehat{d^*}\left(\mathbf{x}_{c_k^j}^t ,X^{(k)}\right)\le \varepsilon \left(2 + \frac{2B(N(t))M(N(t))}{w_{m(N(t))}\delta_{\min}}\right)
\end{eqnarray}
for all $l = 1,\ldots,\kappa$.

Now we explain how to use (\ref{conv_proof}) to prove the asymptotic consistency of Algorithm \ref{algo::online_known_k}. Consider an index $i \in G_{k'}$ for some $k' \in \{1,\ldots,\kappa\}$. Then on one hand, using (\ref{upper2}) and (\ref{upper3}), we get for $k\in\{1,\ldots,\kappa\}$, $k\neq k'$,
\begin{eqnarray}
\label{lower:dhat}
&&\frac{1}{\eta^t} \sum_{j=1}^{N(t)} w_j \gamma_j^t \widehat{d^*}\left(\mathbf{x}_i^t,\mathbf{x}_{c_k^j}^t\right)\nonumber\\
&&\geq \frac{1}{\eta^t} \sum_{j=1}^{N(t)} w_j \gamma_j^t \widehat{d^*}\left(\mathbf{x}_i^t ,X^{(k)}\right) - \frac{1}{\eta^t} \sum_{j=1}^{N(t)} w_j \gamma_j^t \widehat{d^*}\left(\mathbf{x}_{c_k^j}^t ,X^{(k)}\right) \nonumber \\
&& \geq \frac{1}{\eta^t} \sum_{j=1}^{N(t)} w_j \gamma_j^t \left(d^*\left(X^{(k)},X^{(k')}\right)- \widehat{d^*}\left(\mathbf{x}_i^t, X^{(k')} \right)\right)\nonumber\\
&&\hspace{2cm}- \frac{1}{\eta^t} \sum_{j=1}^{N(t)} w_j \gamma_j^t \widehat{d^*}\left(\mathbf{x}_{c_k^j}^t ,X^{(k)}\right) \nonumber \\
&&\geq \delta_{\min} - 2\varepsilon \left(2 + \frac{2B(N(t))M(N(t))}{w_{m(N(t))} \delta_{\min}}\right).
\end{eqnarray}
On the other hand, for any $N\in\mathbb N$, by using the wide-sense ergodicity, there is $T(N)$ such that for all $t\ge T(N)$,
\begin{align}
\label{left_bound}
 \max_{\substack{k \in \{1,\ldots,\kappa\} \\ i \in G_k \cap \left \{ 1,\ldots,N \right \}} } \widehat{d^*}\left(\mathbf{x}_i^t, X^{(k)}\right) \le\varepsilon.
\end{align}
Since $\varepsilon$ can be arbitrarily chosen, it follows from (\ref{lower:dhat}) and (\ref{left_bound}) that
\begin{align}
\label{final:bound}
\argmin_{k \in \{1,\ldots,\kappa\}} \frac{1}{\eta^t}  \sum_{j=1}^{N(t)} w_j \gamma_j \widehat{d^*}\left(\mathbf{x}_i^t ,\mathbf{x}_{c_k^j}^t\right) = k'
\end{align}
holds almost surely for all $i=1,\ldots,N$ and all $t\ge \max\{T,T(N)\}$. 
Theorem \ref{thm:online} is proved.
\end{proof}
The next part involves discussion of the complexity costs of the above two algorithms. 
\begin{enumerate}
\item For offline setting, our Algorithm \ref{algo::offline_known_k} requires $N(N-1)/2$ calculations of $\widehat{d^*}$, against $\kappa N$ calculations of $\widehat{d}$ in the offline algorithm in \cite{khaleghi2016}. In each $\widehat{d^*}$, the matrices distance $\rho^*$ consists of $m_n^2$ calculations of Euclidean distances. Then iterating over $m,l$ in $\widehat{d^*}$ we see that at most $\mathcal O(nm_n^3)$ computations of Euclidean distances, against $\mathcal O(nm_n/|\log s|)$ computations of $\hat d$ for the offline algorithm in \cite{khaleghi2016}, where
$$
s=\min_{\substack{ X_i^{(1)}\neq X_j^{(2)} \\ i\in\{1,\ldots,n_1\};j\in\{1,\ldots,n_2\}}}\left|X_i^{(1)}-X_j^{(2)}\right|.
$$
It is known that efficient searching algorithm can be utilized to determine $s$,  with at most $\mathcal O(n\log(n))$ ($n=\min\{n_1,n_2\}$) computations. Therefore our Algorithm \ref{algo::offline_known_k} is computationally competitive to the one in \cite{khaleghi2016}.
\item For online setting, we can hold a similar discussion as in \cite{khaleghi2016}, Section 5.1. There it shows the computational complexity of updates of $\widehat{d^*}$ for both our Algorithm \ref{algo::online_known_k} and the online algorithm in \cite{khaleghi2016} is at most $\mathcal O(N(t)^2+N(t)\log^3n(t))$ (here we take $m_{n(t)}=\lfloor\log n(t)\rfloor$). Therefore the overall difference of computational complexities between the 2 algorithms are reflected by the complexity of computing $\widehat{d^*}$ and $\widehat{d}$ (see Point 1).
\end{enumerate}
\subsection{Efficient Dissimilarity Measure}
\citet{kleinberg2003} presented a set of three simple properties that a good clustering function should have: \textit{scale-invariance}, \textit{richness} and \textit{consistency}. Further, he demonstrated that there is no clustering function that satisfies these properties at the meanwhile. He pointed out, as one particular example, that the centroid-based clustering basically does not satisfy the above \textit{consistency} property (note that this is a different concept from our asymptotic consistency). In this section we show that, although the consistency property is not satisfied, there exists some other criterion of efficiency of dissimilarity measure in a particular setting. It is the so-called \textit{efficient dissimilarity measure}.
\begin{definition}[Efficient dissimilarity measure]
\label{efficientd}
Assume that the samples $S=\{\mathbf x(\xi):~\xi\in \mathcal H\}$ ($\mathcal H\subset\mathbb R^q$ for some $q\in\mathbb N$), meaning that all the paths $\mathbf x(\xi)$ are indexed by a set of real-valued parameters $\xi$. Then a clustering function is called efficient if its dissimilarity measure $d$ satisfies that, there exists $c>0$ so that for any $\mathbf x(\xi_1),\mathbf x(\xi_2)\in S$,
$$
d(\mathbf x(\xi_1),\mathbf x(\xi_2))=c\|\xi_1-\xi_2\|,
$$
where $\|\cdot\|$ denotes some norm defined over $\mathbb R^q$.
\end{definition}
Mathematically, efficient dissimilarity measure is a metric induced by some norm. Clustering processes based on efficient dissimilarity measure will then be equivalent to clustering under classical distances in $\mathbb R^q$, such as Euclidean distance, Manhattan distance, or Minkowski distance. The latter setting has well-known advantages in cluster analysis. For example,  Euclidean distance performs well when deployed to datasets that include compact or isolated clusters \citep{Jain1996,Jain1999}; when the shape of clusters is hyper-rectangular \citep{Xu2005}, Manhattan distance can be used; Minkowski distance, including Euclidean and Manhattan distances as its particular cases, can be utilized to solve clustering obstacles \citep{Wilson1997}. There is a rich literature on comparing the above three distances to each other through discussing of their advantages and inconveniences. We refer to \citet{Shirkhorshidi2015} and the references therein.

In the next section we present an excellent example, to show how to improve the efficiency of our consistent algorithms, for clustering self-similar processes with wide-sense stationary ergodic increments.
\section{Self-similar Processes and Logarithmic Transformation}
\label{sec:log}
In this section we introduce a non-linear transformation of the covariance matrices in $\widehat{d^*}$, in order to improve the efficiency of clustering. This transformation is based on logarithmic function. We use one example to explain how this transformation works. We show this transformation maps $\widehat{d^*}$ to some covariance-based dissimilarity measure similar to an efficient one, when applied to clustering self-similar processes.

\begin{definition}[Self-similar process, see \cite{Samorodnitsky1994}]
A process $X^{(H)}=\{X_t^{(H)}\}_{t\in T}$ (e.g., $T=\mathbb R$ or $\mathbb Z$) is self-similar with index $H\in(0,1)$ if, for all $n\in\mathbb N$, all $t_1,\ldots,t_n\in T$,  and all $c\neq0$ such that $ct_i\in T$ ($i=1,\ldots,n$), 
$$
\Big(X_{t_1}^{(H)},\ldots,X_{t_n}^{(H)}\Big)\stackrel{law}{=}\Big(|c|^{-H}X_{ct_1}^{(H)},\ldots,|c|^{-H}X_{ct_n}^{(H)}\Big).
$$
\end{definition}
It can be shown that a self-similar process has necessarily zero mean and its  covariance structure is indexed by its self-similarity index $H$, in the following way \citep{Embrechtsu2000}.
\begin{theorem}
\label{cov:self}
Let $\big\{X_t^{(H)}\big\}_{t\in T}$ be a zero-mean self-similar process with index $H\in(0,1)$ and  with wide-sense stationary ergodic increments. Assume $\mathbb E|X_1^{(H)}|^2<+\infty$, then for any $s,t\in T$,
$$
\mathbb Cov\left(X_s^{(H)}, X_t^{(H)} \right)=\frac{\mathbb E|X_1^{(H)}|^2}{2}\left(|s|^{2H}+|t|^{2H}-|s-t|^{2H}\right).
$$
\end{theorem}
The corollary below follows.
\begin{corollary}
Let $\{X_t^{(H)}\}_{t\in T}$ be a zero-mean self-similar process with index $H$ and weakly stationary increments. Assume $\mathbb E|X_1^{(H)}|^2<+\infty$. For $h>0$ small enough, define the increment process $Z_h^{(H)}(s)=X_{s+h}^{(H)}-X_s^{(H)}$, then for $s,t\in T$ such that $s-t\ge h$, we have
\begin{equation}
\label{cov:increment}
\mathbb Cov\left(Z_h^{(H)}(s),Z_h^{(H)}(t) \right)=\frac{\mathbb E|X_1^{(H)}|^2}{2} \left((s-t-h)^{2H}+(s-t+h)^{2H}-2(s-t)^{2H} \right).
\end{equation}
\end{corollary}
Applying three times the mean value theorem to (\ref{cov:increment}) leads to 
\begin{eqnarray}
\label{asym:cov}
&&\mathbb Cov\left(Z_h^{(H)}(s), Z_h^{(H)}(t) \right)=H\mathbb E |X_1^{(H)}|^2\left((v_1^{(H)})^{2H-1}-(v_2^{(H)})^{2H-1}\right)h\nonumber\\
&&=H(2H-1)\mathbb E |X_1^{(H)}|^2(v^{(H)})^{2H-2}h,
\end{eqnarray}
for some $v_1^{(H)}\in(s-t,s-t+h)$, $v_2^{(H)}\in(s-t-h,s-t)$ and $v^{(H)}\in(v_2^{(H)},v_1^{(H)})$.  We see that the item $\mathbb Cov\left(Z_h^{(H)}(s),Z_h^{(H)}(t)\right)$ is a non-linear function of $H$. Next we would find a function $g$ such that $g\left(\mathbb Cov\left(Z_h^{(H)}(s),Z_h^{(H)}(t)\right)\right)$ is linearly dependent of $H$. To this end we introduce the following $\log^*$-transformation: for $x\in\mathbb R$, define
$$
\log^*(x):=\sgn(x)\log|x|=\left\{
\begin{array}{ll}
\log(x)&\mbox{if $x>0$};\\
-\log(-x)&\mbox{if $x<0$};\\
0&\mbox{if $x=0$}.
\end{array}\right.
$$
Introduction to $\log^*$-transformation is driven by the following 2 motivations:
\begin{description}
\item [\textbf{Motivation 1}] The $\log^*$ function transforms the current dissimilarity measure to the one which \enquote{linearly} depends on its variable $H$.
\item [\textbf{Motivation 2}] The value $\log^*(x)$ preserves the sign of $x$, which leads to the consequence that larger distance between $x,y$ yields larger distance between $\log^*(x)$ and $\log^*(y)$. 
\end{description}
Applying $\log^*$-transformation to the covariances of $Z_h^{(H)}$ given in (\ref{asym:cov}), we obtain
\begin{eqnarray*}
&&\log^*\left(\mathbb Cov\left(Z_h^{(H)}(s),Z_h^{(H)}(t)\right)\right)\\
&&=\sgn(2H-1)\left((2H-2)\log v^{(H)}+\log h+\log(H|1-2H|\mathbb Var(X_1^{(H)}))\right).
\end{eqnarray*}
When $v^{(H)}$ and $h$ are small the items $\log v^{(H)}$ and $\log h$ are significantly large so $\log(H|1-2H|\mathbb Var(X_1^{(H)}))$ becomes negligible. Thus we can write
$$
\log^*\left(\mathbb Cov\left(Z_h^{(H)}(s),Z_h^{(H)}(t)\right)\right)\approx\sgn(2H-1)\left((2H-2)\log v^{(H)}+\log h\right).
$$
In conclusion, 
\begin{itemize}
\item When $H_1,H_2\in(0,1/2]$ or $H_1,H_2\in[1/2,1)$, the item $\log^*\left(\mathbb Cov\left(Z_h^{(H)}(s),Z_h^{(H)}(t)\right)\right)$ is \enquote{approximately linear} on $H\in(0,1/2]$ or on $H\in[1/2,1)$. 

Using the approximation $\log v^{(H_1)}\approx \log v^{(H_2)}$ for $H_1,H_2\in(0,1/2]$ or $H_1,H_2\in[1/2,1)$, we have
\begin{eqnarray*}
&&\log^*\left(\mathbb Cov\left(Z_h^{(H_1)}(s),Z_h^{(H_1)}(t)\right)\right)-\log^*\left(\mathbb Cov\left(Z_h^{(H_2)}(s),Z_h^{(H_2)}(t)\right)\right)\\
&&\approx 2\sgn(2H_1-1)(H_1-H_2)\log v^{(H_1)}.
\end{eqnarray*}
\item When $H_1\in(0,1/2]$ and $H_2\in(1/2,1)$, $\log^*\left(\mathbb Cov\left(Z_h^{(H)}(s),Z_h^{(H)}(t)\right)\right)$ turns out to be relatively large, because we have 
\begin{eqnarray*}
&&\log^*\left(\mathbb Cov\left(Z_h^{(H_1)}(s),Z_h^{(H_1)}(t)\right)\right)-\log^*\left(\mathbb Cov\left(Z_h^{(H_2)}(s),Z_h^{(H_2)}(t)\right)\right)\\
&&\approx-(2H_1-2)\log v^{(H_1)}-(2H_2-2)\log v^{(H_2)}\\
&&\ge 2(2-H_1-H_2)\min\left\{\log v^{(H_1)},\log v^{(H_2)}\right\}.
\end{eqnarray*}
\end{itemize}
Taking advantage of the above facts we define the new empirical covariance-based dissimilarity measure (based on the definition (\ref{def:d*_bis})) to be
$$
\widehat{d^{**}}(\mathbf z_{1},\mathbf z_{2}):=\sum_{m= 1}^{m_n} \sum_{l= 1}^{n-m+1} w_m w_l \rho^*\left (\nu^{**}(Z^{(H_1)}_{l\ldots n},m), \nu^{**} (Z^{(H_2)}_{l\ldots n},m)\right),
$$
where $\nu^{**}(Z^{(H_1)}_{l\ldots n},m)$ is the empirical covariance matrix of $Z_h^{(H_1)}$,  $\nu^*(Z^{(H_1)}_{l\ldots n},m)$, with each of its coefficients transformed by $\log^*$: let $M=\{M_{i,j}\}_{i=1,\ldots,m;~j=1,\ldots,n}$ be an arbitrary real-valued matrix, define
$$
\log^*M:=\left\{\log^*M_{ij}\right\}_{i=1,\ldots,m;~j=1,\ldots,n}.
$$
Then we have 
$$
\nu^{**}(Z^{(H_1)}_{l\ldots n},m):=\log^*\left(\nu^{*}(Z^{(H_1)}_{l\ldots n},m)\right).
$$
Now given 2 wide-sense stationary ergodic processes $X^{(1)}$, $X^{(2)}$, we choose $\{w_j\}_{j\in\mathbb N}$ to satisfy
\begin{equation}
\label{condition_2}
\sum_{m,l=1}^{\infty}  w_m w_l \rho^*\left(\log^*(V_{l,l+m-1}(X^{(1)})),\log^*(V_{l,l+m-1}(X^{(2)})\right)<+\infty,
\end{equation}
where we denote by
$$
V_{l,l+m-1}(X^{(1)}):=\mathbb Cov\left(X_l^{(1)},\ldots,X_{l+m-1}^{(1)}\right).
$$
Then define the $\log^*$-transformation of the covariance-based dissimilarity measure to be
\begin{equation}
\label{d**}
d^{**}(X^{(1)},X^{(2)}):=\sum_{m,l=1}^{\infty}  w_m w_l \rho^*\left(\log^*(V_{l,l+m-1}(X^{(1)})),\log^*(V_{l,l+m-1}(X^{(2)})\right).
\end{equation}
Using the fact that $\log^*$ is continuous over $\mathbb R\backslash\{0\}$ and the weak ergodicity of $Z_h^{(H)}$, we have the following version of ergodicity:
$$
\widehat{d^{**}}(\mathbf z_{1},\mathbf z_{2})\xrightarrow[n\to\infty]{a.s.}d^{**}\big(Z_h^{(H_1)},Z_h^{(H_2)}\big).
$$
Unlike $\widehat{d^*}$, the dissimilarity measure $\widehat{d^{**}}$ is approximately linear with respect to the self-similarity index $H$.  Indeed, it is easy to see that
\begin{equation}
\label{d**}
\widehat{d^{**}}(\mathbf z_{1},\mathbf z_{2})\sim\left\{
\begin{array}{ll}
|H_1-H_2|<1,&\mbox{for $H_1,H_2\in(0,1/2]$ or $H_1,H_2\in[1/2,1)$};\\
2(2-H_1-H_2)>1,&\mbox{for $H_1\in(0,1/2)$ and $H_2\in[1/2,1)$},
\end{array}\right.
\end{equation}
where $H_1,H_2$ correspond to the self-similarity indexes of $X^{(H_1)},X^{(H_2)}$ respectively. In fact, from (\ref{d**}) we can say that $\widehat{d^{**}}$ satisfies Definition \ref{efficientd} in the wide sense: it is approximately linearly dependent of $|H_1-H_2|$ when $H_1,H_2$ are in the same group out of $(0,1/2]$ and $[1/2,1)$; it is approximately larger than $|H_1-H_2|$ when $H_1,H_2$ are in different groups out of $(0,1/2]$ and $[1/2,1)$. This fact allows our asymptotically consistent algorithms to be more efficient when clustering self-similar processes with weakly stationary increments, having different values of $H$. In Section \ref{sec::sim_gauss} we provide an example of clustering using our consistent algorithms with and without the $\log^*$-transformation, when the observed paths are from a well-known self-similar process with stationary increments -- fractional Brownian motion.

\section{Simulation and Empirical Study} \label{sec::exper_results}
This section is devoted to applying our clustering algorithms to several synthetic data and real-world data. It is worth noting that, in our statistical setting, the auto-covariance functions are supposed to be unavailable, then the prior choice of the weights $w_j$ presents some trade-off between the convergence of the dissimilarity measure and practical application. On one hand, low rate of convergence (e.g. $w_j=1/j(j+1)$) risks to a divergent dissimilarity measure $d^*$ (see (\ref{def:d*})). On the other hand, high rate of convergence (e.g., $w_j=1/j^3(j+1)^3$) will only make use of some first observations in the sample paths. We believe that the first issue is a minor one in practice, because for most of the wide-sense stationary ergodic processes (especially Gaussian) taking $w_j=1/j(j+1)$ can lead to convergent $d^*$. Also, in practice, instead of (\ref{def:d*}) it is fine to regard
$$
d^*\left(X^{(1)},X^{(2)}\right):= \sum_{m,l = 1}^{N} w_m w_l \rho^*\left(V_{l,l+m-1}(X^{(1)}),V_{l,l+m-1}(X^{(2)})\right),
$$
for some $N$ large enough.

Therefore, through this entire section we take $w_j=1/j(j+1)$ and $m_n=\lfloor \log n\rfloor$ (recall that $\lfloor\cdot\rfloor$ denotes the floor number) in the covariance-based dissimilarity measure $\widehat{d^*}$. Next we explain how to prepare offline and online datasets in this simulation study.
\begin{description}
\item[\textbf{Offline dataset simulation:}] For each scenario, we simulate 5 groups of sample paths, each consists of 10 paths with length $N(t)=5t$, for the time steps  $t=1,2,\ldots,50$. Algorithm \ref{algo::offline_known_k} is performed over 100 such scenarios, and the misclassification rate is calculated.
\item[\textbf{Online dataset simulation:}] For each scenario, we simulate 5 groups of sample paths. Let the total number of sample paths be $N(t) = 30 + \lfloor (t-1)/10 \rfloor$ at each time step $t$. That is, there are 6 sample paths in each of the 5 groups when $t=1$. And the number of sample paths in each group will increase by 1 once the time $t$ increases by 10. For $i=1,2,\ldots$, the $i$th sample path in each group has length $n_i(t) = 5[t-(i-6)^+]$, where $x^+ = \max(x,0)$. 
\end{description}
We then apply the proposed clustering algorithms to both offline and online settings, and determine their corresponding misclassification rates. These misclassification rates are utilized to intuitively illustrate the asymptotic consistency of our clustering algorithms, or to compare the performances of our clustering approaches to other ones.  Recall that the misclassification rate (i.e. mean clustering error rate, see Section 6 in \cite{khaleghi2016}) is obtained by dividing the number of misclassified paths by the total number of paths per scenario, then average all these fractions:
\begin{equation*}
p:=avg\left(\frac{\#~\mbox{of misclassified sample paths}}{\#~\mbox{of total sample paths collected}}\right).
\end{equation*}
More precisely, let $(C_1,\ldots,C_\kappa)$ denote the ground truth clusters of the $N$ sample paths $\mathrm x_1,\ldots, \mathrm x_N$. We define the ground truth cluster labels by
$$
L_k = \underbrace{(k,\ldots,k)}_{\# C_k ~\mbox{times}},~\mbox{for}~k=1,\ldots,\kappa.
$$
Let $(l_1,\ldots,l_N)$ denote the cluster labels of $(\mathrm x_1,\ldots,\mathrm x_N)$ output by some clustering approach. Then the misclassification rate $p$ of this approach is computed by
\begin{equation}
\label{misclassification_rate}
p=\min_{\substack{\sigma\in S_{\kappa}\\(\pi_1,\ldots,\pi_N)=(L_{\sigma(1)},\ldots,L_{\sigma(\kappa)})}}\frac{\sum\limits_{i=1}^N{\mathds 1_{\{\pi_i\ne l_i\}}}}{N},
\end{equation}
where $S_\kappa$ denotes the group of all possible permutations over the set $\{1,\ldots,\kappa\}$.

For example, in one scenario of 7 sample paths, if the ground truth cluster labels of $(\mathrm x_1,\ldots,\mathrm x_7)$ satisfy
$$
\left(L_1, L_2, L_3\right)= ((1,1),(2),(3,3,3,3)),
$$
while the clustering algorithm output cluster labels
 corresponding to $(\mathrm x_1,\ldots,\mathrm x_7)$ are given by 
$$
\left(l_1,\ldots,l_7\right) = \left(2,1,1,2,3,2,1\right),
$$
then according to Eq.  (\ref{misclassification_rate}), the misclassification rate is $4/7$. This can be explained as, at least 4 changes of labels are needed to let the output cluster labels match that of the ground truth ones $(1,1,3,2,2,2,2)$:
$$
l_1\leftarrow{} 1;~l_3\leftarrow{} 3;~l_5\leftarrow{} 2;~l_7\leftarrow{} 2.
$$
We provide the implementation of the misclassification rate (see Eq. (\ref{misclassification_rate})) in MATLAB  publicly online as \textbf{misclassify\_rate.m} \footnote{\url{https://github.com/researchcoding/clustering_stochastic_processes/blob/master/misclassify_rate.m}.}.

\subsection{Clustering Non-Gaussian Discrete-time Stochastic Processes}
\label{sub:simulation}
In \cite{khaleghi2016} a simulation study on a non-Gaussian  strictly stationary ergodic discrete-time stochastic process (see also \cite{shields1996}) has been performed. Since this process has finite covariance structure, it is also wide-sense stationary ergodic. As a result we can test our clustering algorithms over the same dataset and compare their performances to the ones in \cite{khaleghi2016}. Recall that this process $\{X_t\}_{t\in\mathbb N}$ is  generated in the following way. Fix some \textit{irrational-valued} parameter $\alpha \in (0,1)$.
\begin{itemize} [leftmargin=0.8in]
\item[\bf{Step 1.}] Draw a uniform random number $r_0 \in [0,1]$. 
\item[\bf{Step 2.}] For each index $i=1,2,\ldots,N$:
\begin{description}
\item[\bf{Step 2.1.}] Define
$
r_i = r_{i-1} + \alpha - \lfloor r_{i-1} + \alpha \rfloor.
$
\item[\bf{Step 2.2.}] Define
$
X_i = 
\begin{cases}
1 & \qquad \textrm{when } r_i > 0.5, \\
0 & \qquad \textrm{otherwise}.
\end{cases}
$
\end{description}
\end{itemize}
We simulate 5 groups of sample paths $\{X_t\}_{t\in\mathbb N}$ indexed by the irrational values $\alpha_1 = 0.31...$, $\alpha_2 = 0.33...$, $\alpha_3 = 0.35...$, $\alpha_4 = 0.37...$, $\alpha_5 = 0.39...$ ($\alpha_i$, $i=1,\ldots,5$, each is simulated by a longdouble with a long
mantissa, see \cite{khaleghi2016}), respectively. 
\subsubsection{Offline Dataset}
We demonstrate the asymptotic consistency of Algorithm \ref{algo::offline_known_k} by conducting offline clustering on the simulated offline datasets of $\{X_i\}_{i\in\mathbb N}$. 

The valid blue line in Fig. \ref{fig1::offline_known_kappa} illustrates the asymptotic consistency of Algorithm \ref{algo::offline_known_k} through the fact that its misclassification rate decreases as time $t$ increases. Compared to the simulation study over the same dataset in \cite{khaleghi2016}, the misclassification rate provided by our proposed algorithm converges at a comparable speed (see Figure 2 in \cite{khaleghi2016}), even though Algorithm \ref{algo::offline_known_k} aims to cluster \enquote{covariance structures} but not \enquote{process distributions}. 

The dot-dashed red line in Fig. \ref{fig1::offline_known_kappa} presents the performance of Algorithm \ref{algo::online_known_k} and compares its misclassification rates with the ones from Algorithm \ref{algo::offline_known_k}. Applied to offline dataset, the offline algorithm's misclassification rates are consistently lower than the online algorithm, i.e., the offline dataset clustering algorithm performs better than the online dataset clustering algorithm, when dealing with offline datasets.

\begin{figure}[H]
\centering
\includegraphics[scale = 0.6]{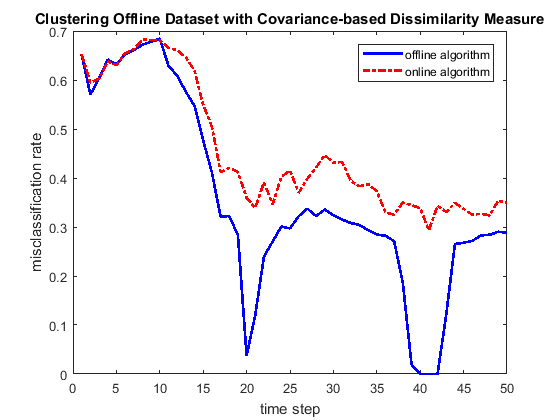}
\caption{The graph compares the misclassification rates of Algorithm \ref{algo::offline_known_k} and Algorithm \ref{algo::online_known_k} applied to offline dataset of non-Gaussian discrete-time processes. 100 runs are performed at each time step $t$ to compute the misclassification rate.}
\label{fig1::offline_known_kappa}
\end{figure}
\subsubsection{Online Dataset}
In our simulated online datasets the number of sample paths and the length of each sample path increase as $t$ increases. This type of setting is mimicking the situation such as modeling financial asset prices, where new assets are launched at each time step. The offline and online clustering algorithms are applied at each time $t$ with $100$ runs, their misclassification rates at each time $t$ are then obtained.
 
Fig. \ref{fig1::online_known_kappa} compares the misclassification rates of offline algorithm and online algorithm applied to the online dataset described above. The periodical pattern, that misclassification rate increases per 10 time steps using offline algorithm, matches the timing of adding new observations. That is, the misclassification rate spikes whenever new observations are obtained. We observe that the misclassification rate of the online algorithm is overall lower than that of offline algorithm in this dataset, reflecting the advantage of online algorithm against the offline one in the case where new observations are expected to occur. It is worth pointing out that our online setting is different from the one in \cite{khaleghi2016}, therefore the two clustering results are not comparable.

Finally, all the codes in MATLAB that reproduce the main conclusions in this subsection can be found publicly online\footnote{\url{https://github.com/researchcoding/clustering_WSSP_with_cov_distance}.}.

\begin{figure}[H]
\centering
\includegraphics[scale = 0.6]{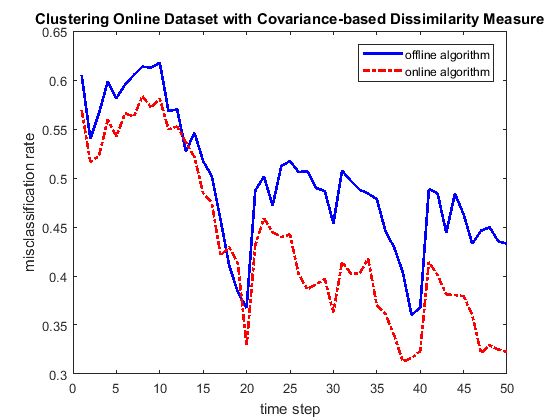}
\caption{The graph compares the misclassification rates of Algorithm \ref{algo::offline_known_k} and Algorithm \ref{algo::online_known_k} applied to online dataset of non-Gaussian discrete-time processes. 100 runs are performed at each time step $t$ to compute the misclassification rate.}
\label{fig1::online_known_kappa}
\end{figure}
\subsection{Clustering Fractional Brownian Motions} \label{sec::sim_gauss}
In this section, we present the performance of proposed offline (Algorithm \ref{algo::offline_known_k}) and online (Algorithm \ref{algo::online_known_k}) methods, on a synthetic dataset sampled from continuous-time Gaussian processes. The wide-sense stationary ergodic processes that we choose are the first order increment processes of \textit{fractional Brownian motions} (see \cite{Mandelbrot1968}). Denote by $\{B^H(t)\}_{t\ge0}$ a fractional Brownian motion with Hurst index $H\in(0,1)$. It is well-known that $B^H$ is a zero-mean self-similar Gaussian process with self-similarity index $H$ and with covariance function
\begin{equation}
\label{cov:BH}
\mathbb Cov\left(B^H(s),B^H(t)\right)=\frac{1}{2}\left(s^{2H}+t^{2H}-|s-t|^{2H}\right),~\mbox{for $s,t\ge0$}.
\end{equation}
Fix $h>0$, define its increment process (with time variation $h$) to be
$$
Z_h^{(H)}(t)=B^H(t+h)-B^H(t),~\mbox{for $t\ge0$}.
$$
$Z_h^{(H)}$ is also called fractional Gaussian noise. Using the covariance function (\ref{cov:BH}) we obtain the auto-covariance function of $Z_h^{(H)}$ below: for $\tau\ge0$,
\begin{equation}
\label{cov:Z}
\gamma(\tau)=\mathbb Cov\left(Z_h^{(H)}(s),Z_h^{(H)}(s+\tau)\right)=\frac{1}{2}\left(|\tau+h|^{2H}+|\tau-h|^{2H}-2|\tau|^{2H}\right).
\end{equation}
Recall that for stationary Gaussian processes such as $Z_h^{(H)}$, the strict ergodicity can be fully expressed in the language of its auto-covariance function $\gamma$, i.e., the following result \citep{Maruyama1970,Slezak2017} provides a sufficient and necessary condition for a stationary Gaussian process to be strictly ergodic.
\begin{theorem}[Strict ergodicity of Gaussian processes]
\label{gaussian:ergodic}
A continuous-time Gaussian stationary process $X$ is strictly ergodic if and only if
\begin{equation}
\label{criterion_ergodic}
\lim_{t\to\infty}\frac{1}{t}\int_0^t|\gamma_X(u)|\ud u=0,
\end{equation}
where $\gamma_X$ denotes the auto-covariance function of $X$.
\end{theorem}
In view of (\ref{cov:Z}) we can deduce that the auto-covariance function $\gamma$ of $Z_h^{(H)}$ satisfies (\ref{criterion_ergodic}). This together with Theorem \ref{gaussian:ergodic} yields that $Z_h^{(H)}$ is second-order strict-sense stationary ergodic, so it is also wide-sense stationary ergodic.

To test our algorithms we simulate $\kappa=5$ groups of independent fractional Brownian paths, with the $i$th group containing $10$ paths as $\{B^{H_i}(1/n),\ldots,B^{H_i}((n-1)/n),B^{H_i}(1)\}$, for the self-similarity indexes
$$
H_1=0.3,~H_2=0.4,~\ldots,~H_5=0.7.
$$
Remark that clustering a zero-mean fractional Brownian motion $B^H$ is equivalent to clustering its increments $Z_{1/n}^{(H)}(t)=B^{H}(t+1/n)-B^H(t)$. These total number of $50$ observed paths of $Z_{1/n}^{(H)}(t)$, each of length $150$, compose an offline dataset and an online one.  The clustering algorithms are applied to the dataset at each time step $t$. 100 runs are made to compute the misclassification rates. we use offline (RESP. online) dataset clustering algorithm to cluster offline (RESP. online) dataset. The purpose is to compare the the algorithms with and without $\log^*$-transformations.

Fig. \ref{fig::sim_results} presents the comparisons of 2 algorithms: one is using the dissimilarity measure $\widehat{d^*}$, the other one is using the dissimilarity measure $\widehat{d^{**}}$,  based on the behavior of misclassification rates as time increases. We conclude that, both algorithms with and without the $\log^*$-transformations are asymptotically consistent. However in both offline and online settings, the covariance-based dissimilarity measure algorithms with $\log^*$-transformation (dashed red lines) have 30\% lower misclassification rates on average than that of algorithms without $\log^*$-transformation (solid blue lines). This simulation study proves the necessity of utilizing $\log^*$-transformed covariance-based dissimilarity measure when the underlying observations have nonlinear, especially power based, covariance-based dissimilarity measure, such as observations sampled from self-similar processes.

The codes in MATLAB used in this subsection are provided publicly online\footnote{\url{https://github.com/researchcoding/clustering_stochastic_processes}.}.

\begin{figure}[H]
\centering
\includegraphics[scale = 0.25]{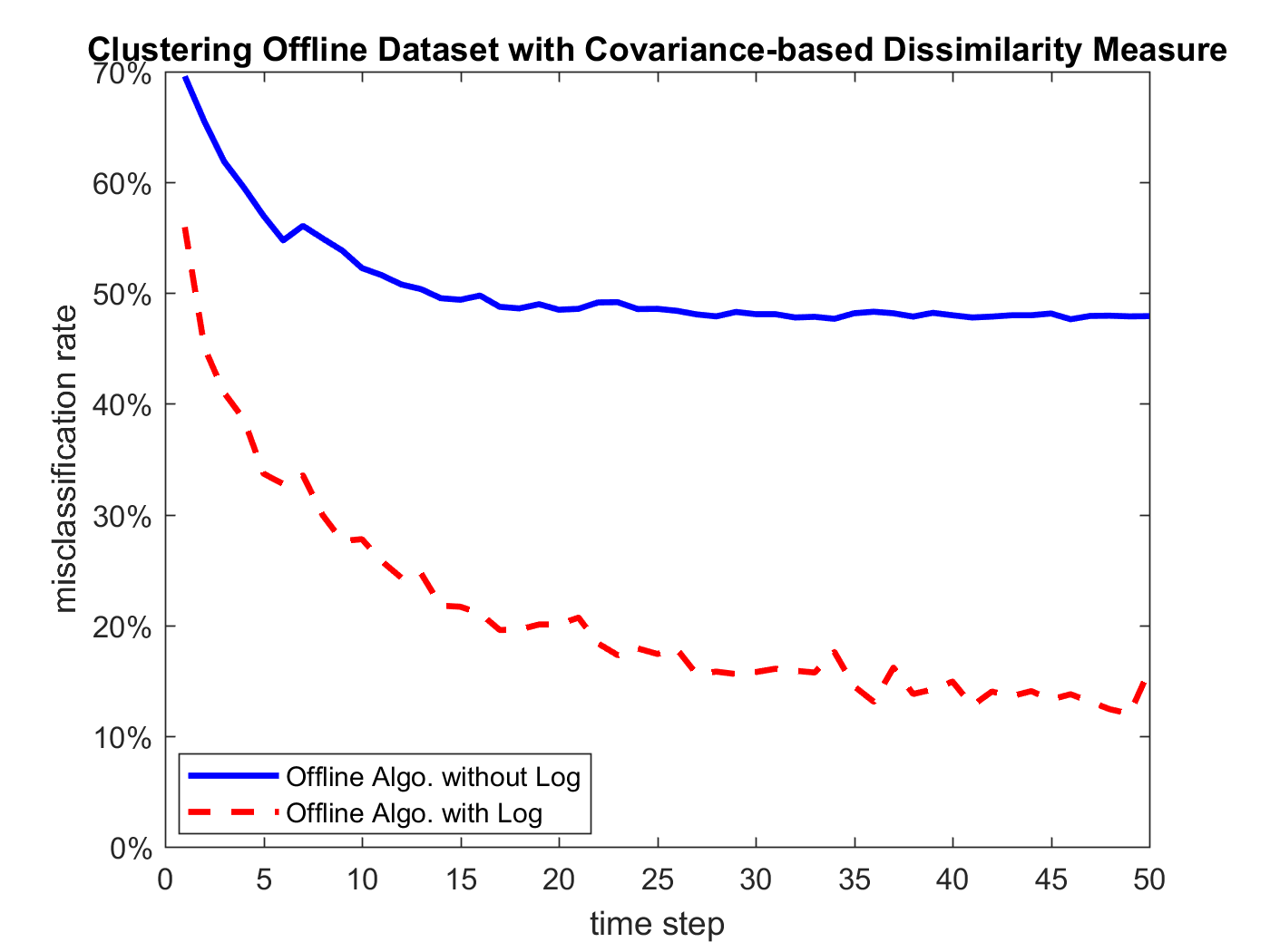} 
\newline
\includegraphics[scale = 0.25]{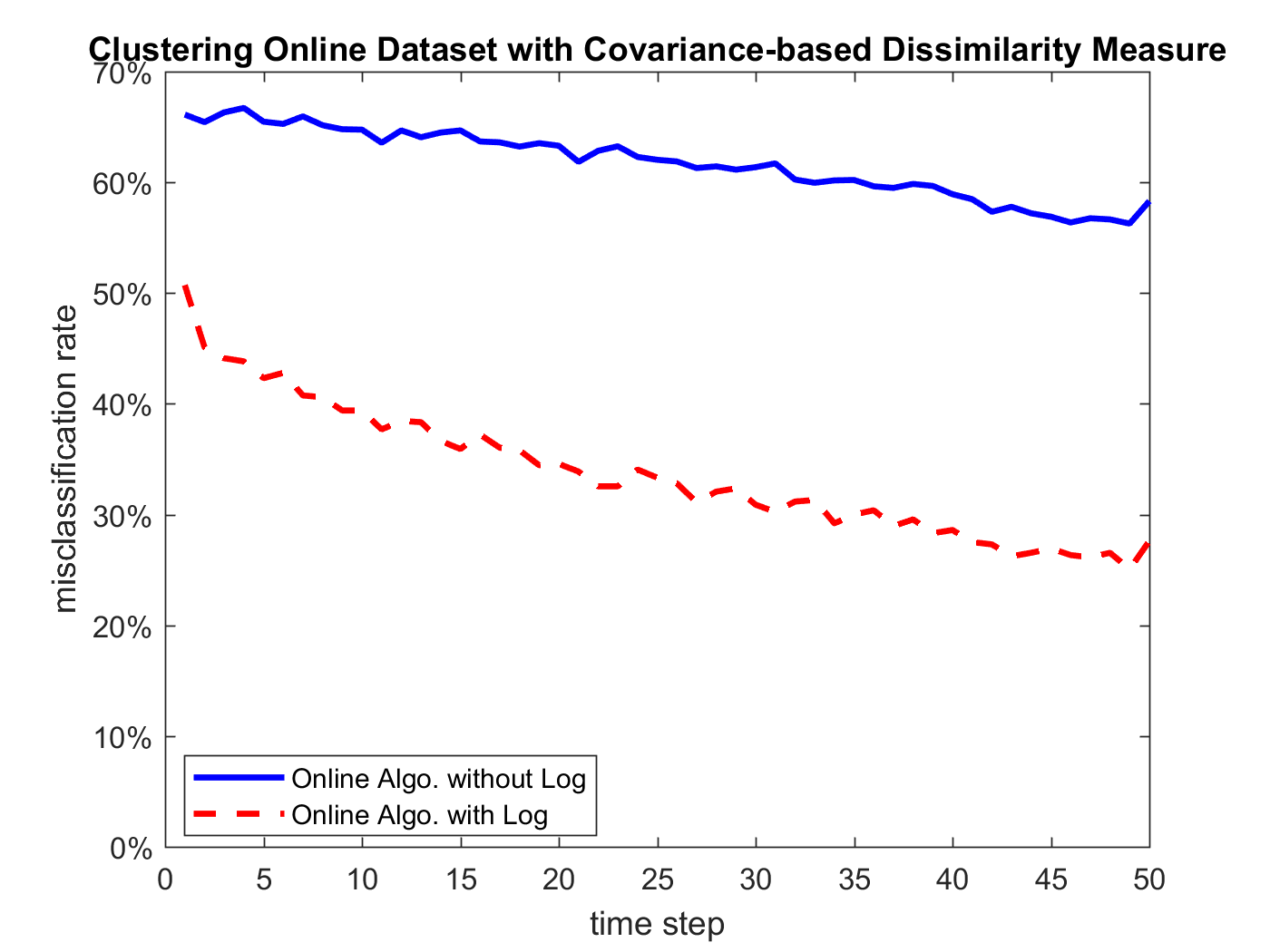}
\newline
\caption{The top graph illustrates the misclassification rates by offline algorithm applied to offline datasets of increments of fractional Brownian motions. The bottom graph plots misclassification rates by online algorithm applied to online datasets.}
\label{fig::sim_results}
\end{figure}

\subsection{Clustering $AR(1)$ Processes: Non Strict-sense Stationary Ergodic} \label{sec::sim_non_gauss}
To show that our algorithms can be applied to clustering non strict-sense stationary ergodic processes, we consider a simulation study on the non-Gaussian $AR(1)$ process $\{Y(t)\}_t$ defined in Example 2, Eq. (\ref{AR1}).  We then conduct the cluster analysis with $\kappa=5$, and specify the values of $a$ in  Eq. (\ref{AR1}) as
\begin{equation*}
a_1 = -0.4, \ a_2 = -0.15, \ a_3 = 0.1, \ a_4 = 0.35, \ a_5 = 0.6.
\end{equation*}
We mimic the procedure in Section \ref{sec::sim_gauss} to generate the offline and online datasets of $\{X(t)\}_t$. Fig. \ref{fig::sim_results_nonGau} illustrates the consistent converging property of offline algorithm and online algorithm under different dataset settings. 

All the codes in MATLAB that reproduce the main conclusions in this subsection can be found publicly online\footnote{\url{https://github.com/researchcoding/clustering_nonGaussian_processes}.}.

\begin{center}
\begin{figure}[H]
\centering
\includegraphics[scale = 0.25]{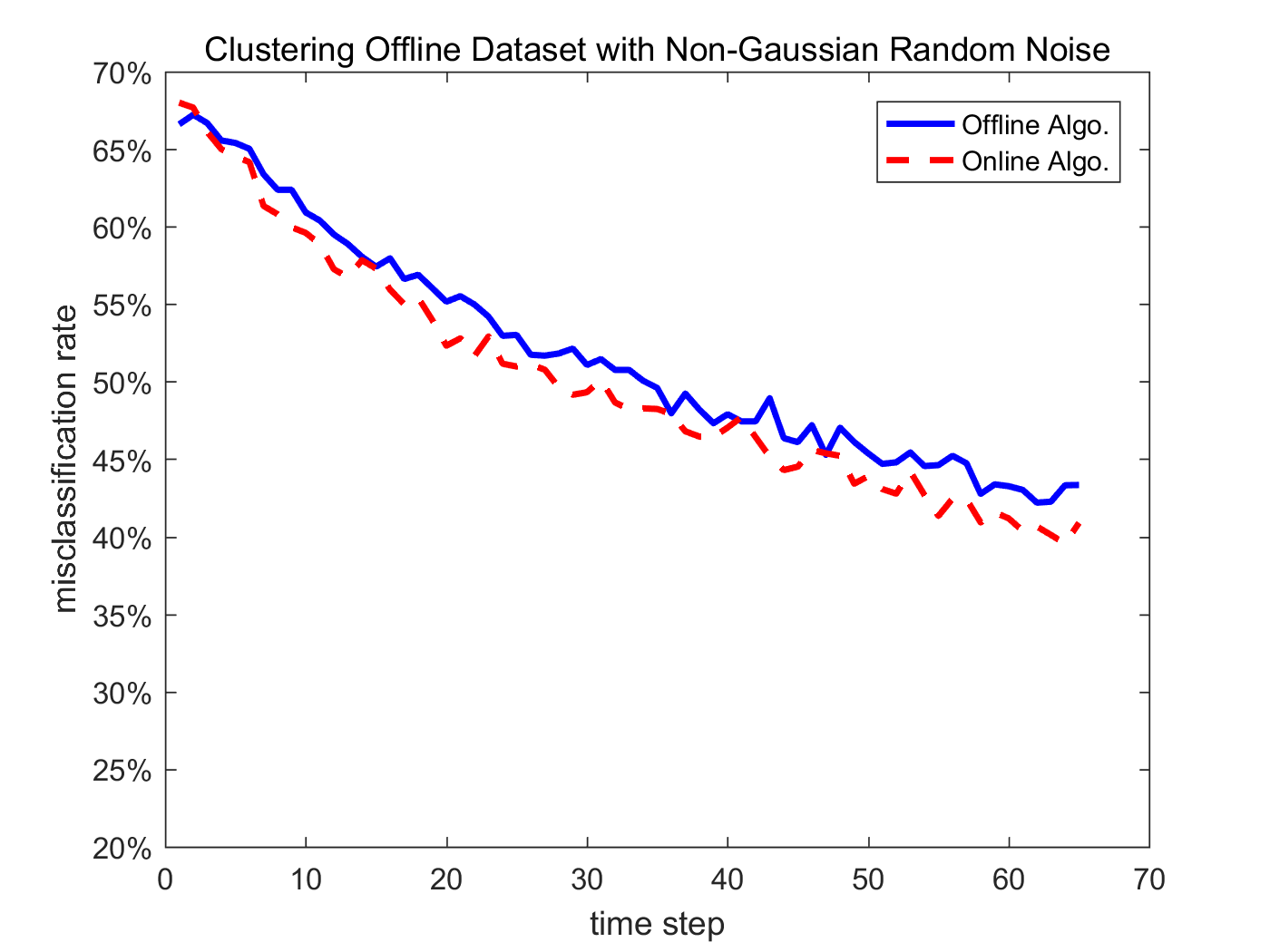} 
\newline
\includegraphics[scale = 0.25]{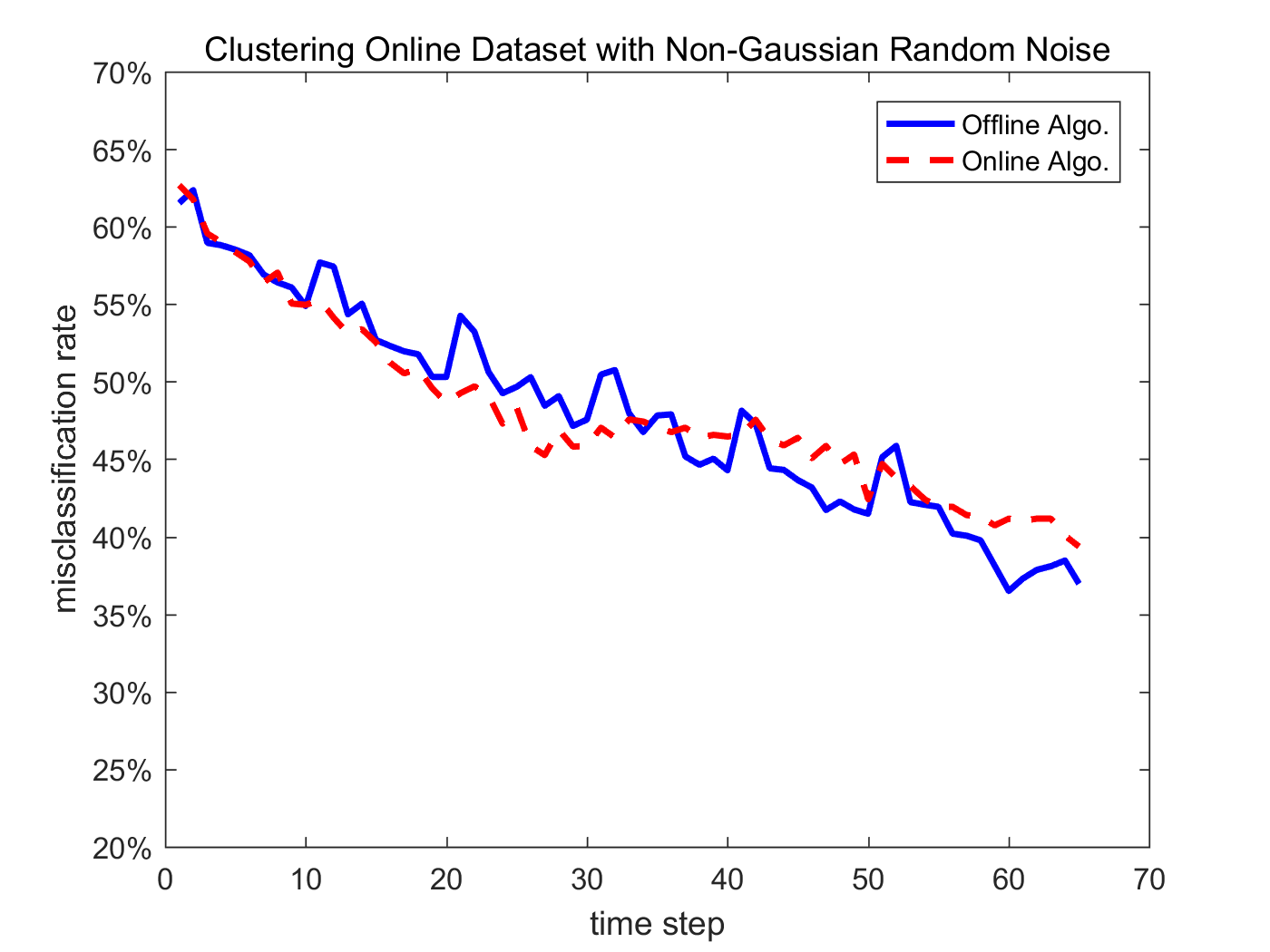}
\newline
\caption{The top graph plots the misclassification rates of ($\log^*$) covariance-based dissimilarity measure along with the increase of time using offline and online algorithms on offline dataset. The bottom graph shows misclassification rates with both algorithms on online dataset.}
\label{fig::sim_results_nonGau}
\end{figure}
\end{center}

\subsection{Application to the Real World: Clustering Global Equity Markets}
\subsubsection{Data and Methodology}
In this section we apply the clustering algorithms to real-world datasets. The application involves in dividing equity markets of major economic entities in the world into different subgroups. In financial economics, researchers usually cluster global equity markets according to either geographical region or the development stage of the underlying economic entities. The reasoning of these clustering methods is that entities with less geographical distance and closer development level involve in more bilateral economic activities. Impacted by similar economic factors, entities with less ``distance'' tend to have higher correlation in stock market performance. This correlation then measures the level of ``comovement'' of stock market indexes on global capital market. 

However, the globalization is breaking the barriers of region and development level. For instance, in 2016 China became the largest trader partner with the U.S. (besides EU)\footnote{Source: U.S. Department of Commerce, Census Bureau, Economic Indicators Division.}. China is not a regional neighbor of the U.S., and is categorized as a developing country by World Bank, in opposite to the U.S. as a developed country. 

We cluster the equity markets in the world according to the empirical covariance structure of their performance, using Algorithms \ref{algo::offline_known_k} and \ref{algo::online_known_k} as purposed in this paper. Then we compare our clustering results with the traditional clustering methodologies. The index constituents of MSCI ACWI (All Country World Index) are selected as the sample data. Each of the observations is a sample path representing the historical monthly returns of underlying economic entities. Through empirical study it is proved that these indexes returns exhibit the \enquote{long memory} path feature hence they can be modeled by self-similar processes such as fractional Brownian motions (see e.g. \cite{Comte1998,Bianchi2008}). Therefore similar to Section \ref{sec::sim_gauss} we may cluster the increments of the indexes returns with the $\log^*$-transformed dissimilarity measure $\widehat{d^{**}}$.  MSCI ACWI is the leading global equity market index and has \$3.2 billion in underlying market capitalization\footnote{As of June 30, 2017, as reported on September 30, 2017 by eVestment, Morningstar and Bloomberg.}. MSCI ACWI contains 23 developed markets, 24 emerging markets from 4 regions: Americas, EMEA (Europe, Middle East and Africa), Pacific and Asia. Table \ref{table::msci_data} lists all markets included in this empirical study. We exclude Greece market due to its bankruptcy after the global financial crisis. 

We construct both offline and online datasets starting from different dates. For offline dataset we let it start from Jan. 30, 2009 to exclude the financial crisis period in 2007 and 2008. This is because, under global stock market crisis, the (downside) performance of equity market is contagious and thus blurs the cluster analysis. The online dataset starts on Jan. 31, 1989, which covers 1997 Asian financial crisis, 2003 dot-com bubble and 2007 subprime mortgage crisis. Another key feature is that 14 markets are added to the MSCI ACWI index (at different time) since 1989, including 1 developed market and 13 emerging markets. Therefore, the case where new time series are observed is handled in online dataset. 

\begin{sidewaystable}[htbp]
\centering
\caption{The categories of major equity markets in the MSCI ACWI (All Country World Index). There are 23 markets from developed economic entities, and 24 markets from emerging countries or areas. The geographical clustering contains Americas, EMEA (Europe, Middle East and Africa), Pacific and Asia. \label{table::msci_data}}
\vspace{0.5em}
\begin{tabular}{ccccccc}
 \hline\hline
 \multicolumn{3}{c}{Developed Markets} & & \multicolumn{3}{c}{Emerging Markets} \\
 \cmidrule(r){1-3} \cmidrule(r){5-7}
Americas	&	Europe \& Middle East	&	Pacific	& &	Americas	&	Europe \& Middle East \& Africa	&	Asia	\\
 \cmidrule(r){1-3} \cmidrule(r){5-7}
Canada	&	Austria	&	Australia &	&	Brazil	&	Czech Republic	&	China (Mainland)	\\
USA	&	Belgium	&	 Hong Kong	& &	Chile	&	Greece	&	India	\\
	&	Denmark	& Japan	&	& Colombia	&	Hungary	&	Indonesia	\\
	&	Finland	&	New Zealand &	&	Mexico	&	Poland	&	Korea	\\
	&	France	&	Singapore &	&	Peru	&	Russia	&	Malaysia	\\
	&	Germany	&		& &		&	Turkey	&	Pakistan	\\
	&	Ireland	&		& &		&	Egypt	&	Philippines	\\
	&	Israel	&		& &		&	South Africa	&	Taiwan	\\
	&	Italy	&		& &		&	Qatar	&	Thailand	\\
	&	Netherlands	&		& &		&	United Arab Emirates	&		\\
	&	Norway	&		& &		&		&		\\
	&	Portugal	&		& &		&		&		\\
	&	Spain	&		& &		&		&		\\
	&	Sweden	&		& &		&		&		\\
	&	Switzerland	&		& &		&		&		\\
	&	United Kingdom	&		& &		&		&		\\
 \hline \hline
\end{tabular}
\begin{flushleft} 
\scriptsize\textbf{Source:} {MSCI ACWI (All Country World Index) market allocation. \url{https://www.msci.com/acwi}.}\\
\end{flushleft}
\end{sidewaystable}

\subsubsection{Clustering Results}
We compare the clustering outcomes of both offline and online datasets with separations suggested by region (4 groups) and development level (2 groups). The factor with the lowest misclassification rate is proved to be the corresponding factor that contributes to increase covariance-based dissimilarity measure the most. In other words, this corresponding factor leads to the clustering of stock markets with the most significant impact. 

Table \ref{table::emp_results} shows that the misclassification rates for development levels are significantly and consistently lower than that of geographical region, for both algorithms (offline and online algorithms) and datasets (offline and online datasets). The clustering results seem to infer that the geographical distance is less dominating than the development level of underlying economic entities, when analyzing different groups of equity markets.

\begin{table}[htbp]
\centering
\caption{The misclassification rates of clustering algorithms on datasets, comparing to clusters suggested by geographical region and development levels. \label{table::emp_results}}
\vspace{0.2em}
\begin{tabular}{ccccc}
 \hline\hline
 & \multicolumn{2}{c}{offline algorithm} & \multicolumn{2}{c}{online algorithm} \\
 \cmidrule(r){2-3} \cmidrule(r){4-5}
	&	region	&	development level	&	region	&	development level	\\
\cmidrule(r){2-3} \cmidrule(r){4-5}
offline dataset	&	63.04\%	&	28.26\%	&	60.87\%	&	23.91\%	\\
online dataset	&	59.57\%	&	44.68\%	&	57.45\%	&	38.30\%	\\
 \hline \hline
\end{tabular}
\end{table}

The global minimum of the misclassification rate occurs when we use online algorithm on offline dataset. Table \ref{table::emp_clu_results} presents the detailed clustering outcome under this circumstance. In each group, the correctly and incorrectly categorized equity markets are listed respectively. For instance, China (Mainland) market is correctly categorized along with other emerging market. Meanwhile Austria market, though being developed market in MSCI ACWI, is categorized to the group where most of the equity markets are emerging markets. The misclassified markets in the emerging group are Austria, Finland, Italy, Norway and Spain markets. The misclassified markets in the developed group are Malaysia, Philippines, Taiwan, Chile and Mexico markets. These empirical results thus suggest that several capital markets have irregular post-crisis performance which blurs the barrier between emerging and developed markets.

\begin{table}[htbp]
\centering
\caption{The clustering outcome of equity markets using offline dataset (starting from Jan. 30, 2009) and online algorithm. The algorithm divides the whole dataset (excluding Greece) into two groups, and in each group the correctly and correctly separated markets are listed, respectively. \label{table::emp_clu_results}}
\vspace{0.2em}
\begin{tabular}{cc|cc}
 \hline\hline
 \multicolumn{2}{c}{Group 1 (Emerging Markets)} & \multicolumn{2}{c}{Group 2 (Developed Markets)} \\
 \cmidrule(r){1-2} \cmidrule(r){3-4}
Correct	&	Incorrect	&	Correct	&	Incorrect	\\ \hline
China (Mainland)	&	Austria	&	Belgium	&	Malaysia	\\
India	&	Finland	&	Denmark	&	Philippines	\\
Indonesia	&	Italy	&	France	&	Taiwan	\\
Korea	&	Norway	&	Germany	&	Thailand	\\
Pakistan	&	Spain	&	Ireland	&	Chile	\\
Brazil	&		&	Israel	&	Mexico	\\
Colombia	&		&	Netherlands	&		\\
PERU	&		&	Portugal	&		\\
Czech Republic	&		&	Sweden	&		\\
Hungary	&		&	Switzerland	&		\\
Poland	&		&	United Kingdom	&		\\
Russia	&		&	Australia	&		\\
Turkey	&		&	Hong Kong	&		\\
Egypt	&		&	Japan	&		\\
South Africa	&		&	New Zealand	&		\\
Qatar	&		&	Singapore	&		\\
United Arab Emirates	&		&	Canada	&		\\
	&		&	USA	&		\\
 \hline \hline
\end{tabular}
\end{table}

The contribution of this real-world dataset cluster analysis is two-fold. First, we explored and determined the principal force that brings structural difference in global capital markets, which potentially predicts the \enquote{comovement} pattern of future index performance. Second, we provided new evidence on the impact of globalization on breaking geographical barriers between economic entities.

\section{Conclusion and Future Perspectives}
Inspired by \citet{khaleghi2016}, we introduce the problem of clustering wide-sense stationary ergodic processes. A new covariance-based dissimilarity measure is proposed to obtain asymptotically consistent clustering algorithms for both offline and online settings. The recommended algorithms are competitive for at least two reasons: 
\begin{enumerate}
\item Our algorithms are applicable to clustering  a wide class of stochastic processes, including any strict-sense stationary ergodic processes whose covariance structures are finite. 
\item Our algorithms are efficient enough in terms of their computational complexity cost. In particular, a so-called $\log^*$-transformation is introduced to improve the efficiency of clustering, for self-similar processes.  
\end{enumerate}
The above advantages have been supported through the simulation study on  non-Gaussian discrete-time processes, fractional Brownian motions, non-Gaussian non strict-sense stationary ergodic $AR(1)$ processes, and a real-world application: clustering global equity markets. The implementations in MATLAB of our clustering algorithms are provided publicly online.

Finally we note that, the clustering framework proposed in our paper focuses on the cases where the true number of clusters $\kappa$ is known. The case for which $\kappa$ is unknown is still open and left to future research. Another interesting problem is that, many stochastic processes are not wide-sense stationary but they get a tight relationship with the wide-sense stationarity. For example, a self-similar process does not necessarily have wide-sense stationary increments, but their Lamperti transformations are strict-sense stationary \citep{Lamperti1962}; locally asymptotically self-similar processes are generally not self-similar but their tangent processes are self-similar \citep{Boufoussi2008}. Our cluster analysis sheds light on clustering the above processes. These topics can be left for future research. \\ \\


\bibliographystyle{spbasic}
\bibliography{ml}
\end{document}